\newcommand{\CLASSINPUTbaselinestretch}{1} 
\newcommand{\CLASSINPUTinnersidemargin}{0.4in} 
\newcommand{\CLASSINPUToutersidemargin}{0.4in} 
\newcommand{\CLASSINPUTtoptextmargin}{0.5in}   
\newcommand{\CLASSINPUTbottomtextmargin}{0.5in}

\documentclass[9.5pt,journal,compsoc]{IEEEtran}

\ifCLASSOPTIONcompsoc
  \usepackage[nocompress]{cite}
\else
  \usepackage{cite}
\fi
\ifCLASSINFOpdf
   \usepackage[pdftex]{graphicx}
\else
   \usepackage[dvips]{graphicx}
\fi
\usepackage[cmex10]{amsmath}
\usepackage{array}

\ifCLASSOPTIONcompsoc
  \usepackage[caption=false,font=footnotesize,labelfont=sf,textfont=sf]{subfig}
\else
  \usepackage[caption=false,font=footnotesize]{subfig}
\fi
\usepackage{url}

\usepackage{amsthm}
\usepackage{multirow}
\usepackage{booktabs}
\newtheorem{theorem}{Theorem}
\newtheorem{definition}{Definition}

\begin{document}
%
\title{Local Distribution in Neighborhood for Classification}
%
%
%
%

\author{Chengsheng Mao$^*$, Bin Hu, Lei Chen, Philip Moore and Xiaowei Zhang

\IEEEcompsocitemizethanks{\IEEEcompsocthanksitem  Chengsheng Mao, Bin Hu, Philip Moore and Xiaowei Zhang are with the School of Information Science and Engineering, Lanzhou University, Gansu, China. 
\IEEEcompsocthanksitem  Corresponding to Chengsheng Mao ( e-mail: chshmao@gmail.com ).
} 
}

\IEEEtitleabstractindextext{%
\begin{abstract}
The \textit{k}-nearest-neighbor method performs classification tasks for a query sample based on the information contained in its neighborhood. Previous studies into the \textit{k}-nearest-neighbor algorithm usually achieved the decision value for a class by combining the support of each sample in the neighborhood. They have generally considered the nearest neighbors separately, and potentially integral neighborhood information important for classification was lost, e.g. the distribution information. This article proposes a novel local learning method that organizes the information in the neighborhood through local distribution. In the proposed method, additional distribution information in the neighborhood is estimated and then organized; the classification decision is made based on maximum posterior probability which is estimated from the local distribution in the neighborhood. Additionally, based on the local distribution, we generate a generalized local classification form that can be effectively applied to various datasets through tuning the parameters. We use both synthetic and real datasets to evaluate the classification performance of the proposed method; the experimental results demonstrate the dimensional scalability, efficiency, effectiveness and robustness of the proposed method compared to some other state-of-the-art classifiers. The results indicate that the proposed method is effective and promising in a broad range of domains.
\end{abstract}

\begin{IEEEkeywords}
Classification, nearest neighbors, local distribution, posterior probability.
\end{IEEEkeywords}
}

\maketitle

\IEEEdisplaynontitleabstractindextext

%
\IEEEpeerreviewmaketitle

\ifCLASSOPTIONcompsoc
\IEEEraisesectionheading{\section{Introduction}\label{sec:introduction}}
\else
\section{Introduction}\label{sec:introduction}
\fi
%
%
%
%
\IEEEPARstart{I}{n} classification problems, a collection of correctly classified samples is usually created as the training set, and classification of each new pattern is achieved using the evidence of samples in the training set. As a classification method, the \textit{k}-Nearest-Neighbor (kNN) algorithm implements the classification task for a query sample using the information of its \textit{k} nearest neighbors (kNNs) in the training set. A simple classification rule generated from the kNNs is the majority voting rule where the query sample is classified to the class represented by the majority of its kNNs. This is the well-known and understood original Voting kNN (V-kNN) rule developed by Cover \& Hart in \cite{cover1967nearest}.

As an instance-based learning algorithm, kNN algorithms have suffered from a number of issues which require resolution to enable its effective use in real-world learning tasks \cite{breiman1984classification}. However, following decades of documented research the majority of these issues have been resolved, or at least, mitigates to a greater or lesser degree \cite{aha1991instance}. Research has moreover identified a number of advantages for the kNN algorithm which has, in practice, been successfully applied to real-world problems and applications. The advantages include: (1) kNN is a non-parametric method and does not require any \textit{a priori} knowledge relating to data distributions; (2) the error rate for kNN methods can approach the optimal Bayes error rate in theory as the sample size tends to infinity \cite{cover1967nearest, duda2012pattern}; (3) kNN can deal with multi-class and multi-lable problems straightforward and effortlessly \cite{zhang2006svm,zhang2007ml}; and (4) the kNN algorithm can be implemented easily due to its relative simplicity.

Due to the advantages identified, the kNN algorithm has been the subject of extensive research and development for use in a range of research domains including Data Mining (DM), Machine Learning (ML) and Pattern Recognition (PR) \cite{geva1991adaptive,wang2014continuous,li2014safe,samanthula2015k}. Especially, kNN has been considered to be one of the top 10 methods in DM \cite{wu2008top} for its usefulness and effectiveness for classification. Also, the kNN algorithm shows its effectiveness in a variety of application areas, including computer vision \cite{boiman2008defense,wang2015visual}, Brain Computer Interface (BCI) \cite{khosrowabadi2011brain}, biometrics \cite{mao2015eeg}, affective computing \cite{li2012improve} and text categorization \cite{tan2006effective,lin2014similarity}. The kNN algorithm provides support for classification problems and usually achieves very good performances in a variety of domains \cite{govindarajan2010evaluation,choi2014secure,mao2015learning,hu2015bayesian}.

While the kNN method has been widely used over several decades due to its effectiveness and simplicity, the traditional V-kNN rule is not guaranteed to be the optimal method when implementation uses only the quantity information of the kNNs in the neighborhood. The organization of information contained in the neighborhood plays a very important role in generating effective and efficient decision rules for kNN algorithms. In this article, we generate the classification rule using the distribution information instead of the quantity information contained in the neighborhood and therefore propose a comprehensive kNN decision rule, termed the Local Distribution based kNN (LD-kNN). The local distribution information in the proposed method would comprehensively consider the quantity information, distance information and the sample position information contained in the neighborhood. Hence, the proposed method is expected to be more effective for classification problems. Our previous work \cite{mao2015nearest} has given an example of LD-kNN using the local Gaussian assumption to estimate the local distribution for classification.

In LD-kNN, the local distribution (around the query sample) of each class is estimated from the samples in the neighborhood; then the posterior probability of the query sample belonging to each class is estimated based on the local distribution. The query sample is assigned to the class with the greatest posterior probability. As the classification is based on the maximum posterior probability, the LD-kNN rule can achieve the Bayes error rate in theory.

Fig. \ref{examples} shows typical cases where the previous kNN methods may fail, while the LD-kNN can be effective through the consideration of local distribution. To validate the LD-kNN method, we consider multivariate data with numerical attributes, and estimate the local distribution for classification. The experimental results using both real and synthetic data sets show that LD-kNN is competitive compared to a number of the state-of-the-art classification methods.

The main contributions of this article can be summarized as follows:
\begin{itemize}
    \item  We introduce the concept of local distribution and first define the local probability density to describe the local distribution, and generate the connection between local distribution and global distribution.
    \item  Based on local distribution, we first propose a generalized local classification formulation (i.e. LD-kNN) that can organize the local distribution information in the neighborhood to improve the classification performance. Through tuning the parameters in LD-kNN, it can be effectively applied to various datasets.
    \item  We implement two LD-kNN classifiers respectively using Gaussian model estimation and kernel density estimation for local distribution estimation and design a series of experiments to research the properties of LD-kNN. And the experimental results demonstrate the superiority of LD-kNN compared to some other state-of-the-art classifiers.
\end{itemize}

The remainder of this article is structured as follows. Section \ref{sec:review} presents a review of kNN methods with a focus on the organization of neighborhood information in Section \ref{subsec:neighborinfo} which would use the idea of proposed method. Section \ref{sec:ldknn} introduces the concept of local distribution and the main idea of LD-kNN. The evaluation and experimental testing are set out in Section \ref{sec:experiments} where results are presented along with a comparative analysis with alternative approaches to pattern classification. Section \ref{sec:discussion} makes a discussion of LD-kNN. The article closes with a conclusion in Section \ref{sec:conclusion}.

\section{A review of kNN}\label{sec:review}
If a dataset has the property that data samples within a class have high similarity in comparison to one another, but are dissimilar to objects in other classes, i.e., near samples can represent the property of a query sample better than more distant samples, the kNN method is considered effective for classification. Fortunately, in the real world, objects in the same class usually have certain similarities. Therefore, kNN method can perform well if the given attributes are logically adequate to describe the class.

However, the original V-kNN is not guaranteed to be the optimal method and improving the kNN for more effective and efficient classification has remained an active research topic over many decades. Effective application of kNN usually depends on the neighborhood selection rule, neighborhood size and the neighborhood information organization rule. And, the efficiency of the kNN method is usually dependent on the reduction of the training data and the search for the neighbors which generally relies on the neighborhood selection rule. Thus, refinements to the kNN algorithm in recent years have focused mainly on four aspects: (1) the data reduction; (2) the neighborhood selection rule; (3) the determination of neighborhood size; and (4) the organization of information contained in the neighborhood, which is mainly concerned in this article.

\subsection{Data reduction}
Data reduction is a successful technique that simultaneously tackles the issues of computational complexity, storage requirements, and noise tolerance of kNN. Data reduction is applied to obtain a reduced representation of the data set that can closely maintain the property of the original data. Thus, the reduced training dataset should require much less storage and be much more efficient for neighbor search. The data reduction usually includes attribute reduction and instance reduction. As the name implies, attribute reduction aims to reduce the number of attributes while instance reduction tries to reduce the number of instances.

\subsubsection{Attribute reduction}
Attribute reduction can include deleting the irrelevant, weak relevant or redundant attributes/dimensions (known as attribute/feature subset selection), or constructing a more relevant attribute set from the original attribute set (known as attribute/feature construction).

The basic heuristic methods of attribute subset selection include stepwise forward selection, stepwise backward elimination, and decision tree induction \cite{nowozin2012improved}. And there are a number of advanced method for attribute subset selection, such as ReliefF algorithms \cite{robnik2003theoretical}, minimal-Redundancy-Maximal-Relevance criterion (mRMR)\cite{peng2005feature} and Normalized Mutual Information Feature Selection (NMIFS) \cite{estevez2009normalized}.

Attribute construction, usually known as dimensionality reduction, is applied so as to obtain a reduced or compressed representation of the original attribute set. There are a number of generic attribute construction methods, including linear transforms like Principal Components Analysis (PCA), Singular Value Decomposition (SVD) and Wavelet Transforms (WT) \cite{press2012numerical}, and nonlinear transforms like manifold learning including isometric feature mapping (Isomap) \cite{tenenbaum2000global}, Locally Linear Embedding (LLE) \cite{roweis2000nonlinear} and Laplacian Eigenmaps (LE) \cite{belkin2001laplacian} etc.

\subsubsection{Instance reduction}
Many researchers have addressed the problem of instance reduction to reduce the size of the training data. Pioneer research into the instance reduction was conducted by Hart \cite{hart1968condensed} with his Condensed Nearest Neighbor Rule (CNN) in 1960s, and subsequently by other researchers \cite{wilson1972asymptotic,tomek1976experiment,lowe1995similarity}. Aha et al. \cite{aha1991instance,aha1992tolerating} presented a series of instance based-learning algorithms that aim to reduce the storage requirement and improve the noise tolerance. The instance reduction techniques can include the instance filtering and instance abstraction \cite{veenman2005nearest}.

Instance filtering, also known as prototype selection \cite{pkekalska2006prototype} or instance selection \cite{jankowski2004comparison}, selects a subset of samples that can represent the property of the whole dataset from the original training data. Several instance filtering approaches have been reported like \cite{brighton2002advances,wilson2000reduction,garcia2012prototype,jankowski2004comparison}. In recent years, some more advanced prototype selection methods have been proposed such as Fuzzy Rough Prototype Selection (FRPS) \cite{verbiest2013frps,verbiest2013owa} and prototype selection using mutual information \cite{guillen2010new}.

Instance abstraction, also known as prototype generation, generates and replaces the original data with new artificial data \cite{triguero2012taxonomy}. Most of the instance abstraction methods use merging or divide-and-conquer strategies to set new artificial samples \cite{chang1974finding}, or are based on clustering approaches \cite{bezdek2001nearest}, Learning Vector Quantization (LVQ) hybrids \cite{kohonen1990improved}, advanced proposals \cite{lam2002discovering,lozano2006experimental,impedovo2014novel}, and evolutionary algorithms-based schemes \cite{cervantes2009ampso,triguero2010ipade,triguero2011differential}.

\subsection{Neighborhood selection}
The neighborhood selection rule usually relates to the definition of the ``closeness'' between two samples that used to find the \textit{k} nearest neighbors of a sample in question. The closeness in the kNN method is usually defined in terms of a distance or similarity function.\footnote{Similarity is usually considered the converse of distance \cite{widdows2004geometry}; in this paper we use similarity and distance interchangeably.} For example, Euclidean distance is the most commonly used in any distance-based algorithms for numerical attributes; while Hamming distance is usually used for categorical attributes in kNN, and edit distance for sequences, maximal sub-graph for graphs. There are a number of other advanced distance metrics that can be used in kNN depending on the type of dataset in question, such as Heterogeneous Euclidean-Overlap Metric (HEOM) \cite{wilson1997improved}, Value Difference Metric (VDM) and its variants \cite{stanfill1986toward,wilson1997improved}, Minimal Risk Metric (MRM) \cite{blanzieri99probability}, and Neighborhood Counting Measure (NCM) \cite{wang2006nearest}. There is also a large body of documented research addressing similarity metrics to improve the effectiveness of kNN \cite{steven1991a,yu2008distance,cunningham2009taxonomy,hsu2009design}; however there are no known distance functions which have been shown to perform consistently well in a broad range of conditions \cite{wang2006nearest}.

There are some other neighborhood selection rules that select the \textit{k} neighbors not only based on the distance function. Hastie and Tibshirani \cite{hastie1996discriminant} used a local linear discriminant analysis to estimate an effective metric for computing neighborhoods, where the local decision boundaries were determined from the centroid information, and the neighborhoods shrank in the direction orthogonal to the local decision boundaries, and stretched out in the direction parallel to the decision boundaries. Guo and Chakraborty \cite{guo2010bayesian} proposed a Bayesian adaptive nearest neighbor method (BANN) that could adaptively select the neighborhood according to the concentration of the data around each query point with the help of discriminants. The Nearest Centroid Neighborhood (NCN) proposed by Chaudhuri \cite{chaudhuri1996new} tries to consider both the distance-based proximity and spatial distribution of \textit{k} neighbors, and works very well, particularly in cases with small sample size \cite{altiotanccay2011improving,gou2012local}. In addition, some researchers use reverse neighborhood where the samples have the query sample in their neighborhoods \cite{korn2000influence,tao2004reverse,koh2013finding}.

The kNNs search methods should be another aspect of neighborhood selection for improving the search efficiency. Though effectively could the basic exhaustive search find the kNNs, it would take much running time, especially if the scale of the dataset is huge. KD-tree \cite{bentley1975multidimensional,friedman1977algorithm} is one of the most commonly used methods for kNNs search; it is effective and efficient on lower dimensional spaces, while its performance has been widely observed to degrade badly on higher dimension spaces \cite{weber1998quantitative}. Some other tree structure based methods can also provide compelling performance in some practical applications such as ball tree \cite{uhlmann1991satisfying}, cover tree \cite{beygelzimer2006cover}, Principal Axis Tree \cite{mcnames2001fast} and Orthogonal Search Tree \cite{liaw2010fast} etc.

In many applications, the exact kNNs is not rigorously required, some approximate kNNs may be good substitutes for the exact kNNs. The approximate kNNs search should be more efficient than the exact kNNs search. This relaxation led to a series of important approximate kNNs search techniques such as Locality Sensitive Hashing (LSH) \cite{datar2004locality}, Spill-tree \cite{liu2004investigation} and Bregman ball tree (Bbtree) \cite{cayton2008fast} etc. Some other recent techniques for efficient approximate kNNs search can be found in \cite{muja2009fast,hajebi2011fast,malkov2014approximate,jegou2011product,esmaeili2012fast,har2012approximate}.

\subsection{Neighborhood size}
The best choice of neighborhood size of kNN much depends on the training data, larger values of \textit{k} could reduce the effect of noise on the classification, but make boundaries between classes less distinct \cite{everitt2011miscellaneous}. In general, a larger size of training data requires a larger neighborhood size so that the classification decisions can be based on a larger portion of data. In practice, the choice of \textit{k} is usually determined by the cross-validation methods \cite{tran2005empirical}. However, the cross validation methods would be time consuming for the iterative training and validation. Thus, researchers try to find an adaptive and heuristic neighborhood size selection method for kNN without training and validation.

Wang et al. \cite{wang2005nearest} solved this ``choice-of-k'' issue by an alternative formalism which selects a set of neighborhoods and aggregates their supports to create a classifier less biased by \textit{k}. And another article \cite{wang2006neighborhood} proposed a neighborhood size selection method using statistical confidence, where the neighborhood size is determined by the criterion of statistic confidence associated with the decision rule. In \cite{hall2008choice}, Hall et al. detailed the way in which the value of \textit{k} determines the misclassification error, which may motivate new methods for choosing the value of \textit{k} for kNN. In addition, some other literatures by Guo et al. \cite{guo2003knn,guo2004knn,guo2006using} proposed a kNN model for classification where the value of \textit{k} is automatically determined with different data. Though the kNN model can release the choice of \textit{k} and can achieve good performances, it introduces two other pruning parameters which may have some effects on classification.

\subsection{Neighborhood information organization}\label{subsec:neighborinfo}
In kNN methods, the classification decision for a query sample is made based on the information contained in its neighborhood. The organization of neighborhood information can be described as, given a query sample $X$ and its kNNs from the training set (the \textit{i}th nearest neighbor is denoted by $X_i$ with the class label $l_i$ and distance $d_i$), how to organize these information to achieve an effective decision rule for the class label $l$ of $X$? There are a number of decision rules which can be employed in the previous studies as follows.

\subsubsection{Voting kNN rules}
The V-kNN rule is the most basic rule that uses only the quantity information to create a majority voting rule where the query sample $X$ is assigned to the class represented by the majority of its kNNs. The V-kNN rule can be expressed as Equation \ref{vknn}, where $I(\cdot)$ is an indicator function that returns 1 if its condition is true and 0 otherwise, and $N_C$ is the number of kNNs in class $C$.
\begin{equation}
\text{V-kNN: } l =\arg \max_{C}{\sum_{i=1}^{k}{I(l_i==C)}} = \arg \max_{C}{N_C}
\label{vknn}
\end{equation}

The V-kNN rule may not represent the optimal approach to organize the neighborhood information when using only quantity information. The main drawback of the V-kNN rule is that the kNNs of a query sample are assumed to be contained in a region of relatively small volume; thus the difference among the kNNs can be ignored. In practice, if the neighborhood can not be too small, this differences may not be always negligible, and can become substantial. Therefore, it can be questionable to assign an equal weight to all the kNNs in the decision process regardless of their relative positions to the query sample. In Fig. \ref{example1}, the V-kNN method finds the ten nearest neighbors of the query sample, of which the 4 belong to Class 1 and the remaining 6 somewhat further distant neighbors belong to Class 2. In this case, the V-kNN rule will misclassify the query sample to Class 2 which has more nearest neighbors, regardless of the distances; clearly it would be more reasonable to classify the query sample to Class 1 for which the samples are nearer to the query sample than Class 2.

\begin{figure}[!t]
\centering
\subfloat{\includegraphics[width=0.8\columnwidth]{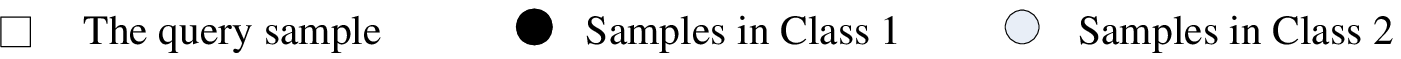}}
\addtocounter{subfigure}{-1}

\subfloat[V-kNN]{\includegraphics[width=0.4\columnwidth]{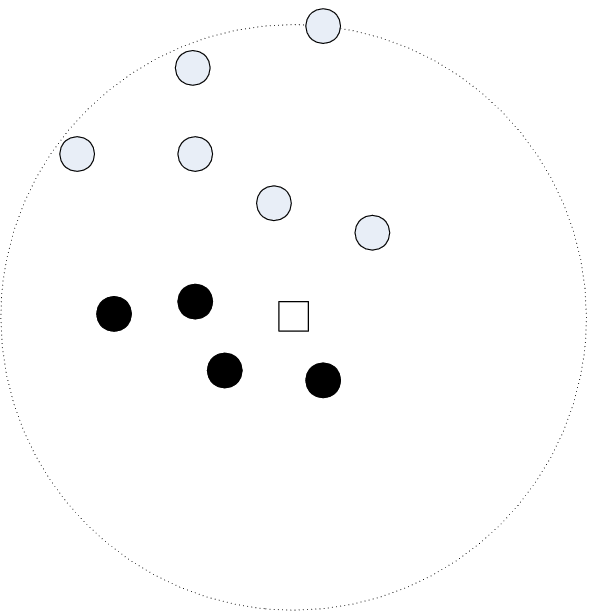}
\label{example1}}
\hspace{0.5cm}
\subfloat[DW-kNN]{\includegraphics[width=0.4\columnwidth]{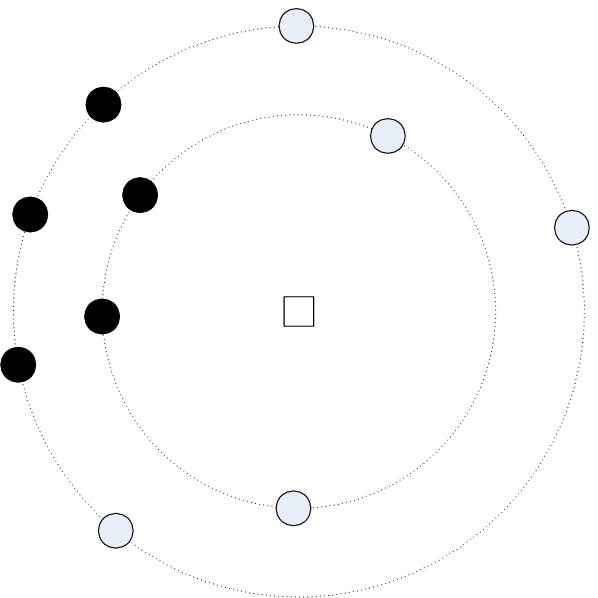}
\label{example2}}

\subfloat[LC-kNN]{\includegraphics[width=0.4\columnwidth]{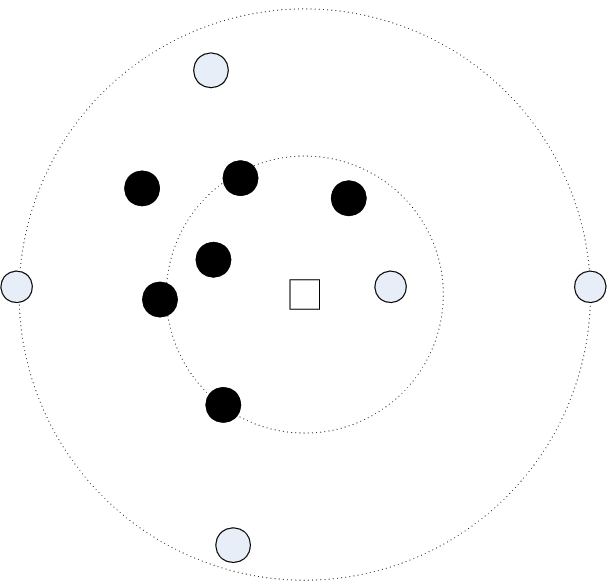}
\label{example3}}
\hspace{0.5cm}
\subfloat[SVM-kNN]{\includegraphics[width=0.4\columnwidth]{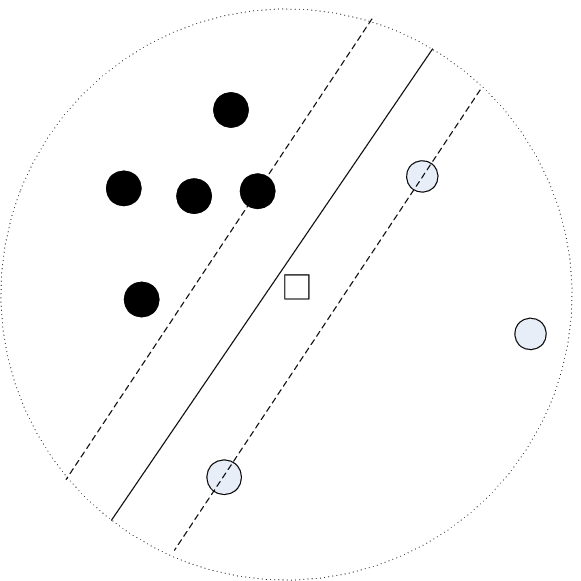}
\label{example4}}

\caption{Some cases show the drawbacks of the corresponding kNN rules.}
\label{examples}

\end{figure}

\subsubsection{Weighted kNN rules}
A refinement to the V-kNN is to apply a weight to each of the kNNs based on its distance to the query point with a greater weight for a closer neighbor; and the query sample is assigned to the class in which the weights of the kNNs sum to the greatest value. This is the Distance-Weighted kNN (DW-kNN) rule. Compared with the V-kNN rule, the DW-kNN rule additionally takes into account the distance information of the kNNs in the decision. If assigning a weight $w_i$ to the \textit{i}th nearest neighbor of $X$, then the DW-kNN rule can be expressed as
\begin{equation}
\text{DW-kNN: } l=\arg \max_{C}{\sum_{i=1}^{k}{I(l_i==C)}\cdot w_i}.
\label{dwknn}
\end{equation}

There is a large body of literature addressing research into the weighting functions for DW-kNN. Dudani \cite{dudani1976distance} has proposed a DW-kNN rule by assigning the \textit{i}th nearest neighbor $x_i$ a distanced-based weight $w_i$ defined as Equation \ref{weight1}. Gou et al. \cite{gou2011improving} have modified Dudani's weighting function using a dual distance weighted voting, which can be expressed as Equation \ref{weight2}. Some other weighting rules can be found at \cite{Hechenbichler2006Weighted,zuo2008kernel}. However, there is no one weighting method that is known to perform consistently well even under some conditions.
\begin{equation}
w_i=
\begin{cases}
\frac{d_k-d_i}{d_k-d_1},  & d_k\neq d_1 \\
1, & d_k=d_1
\end{cases}
\label{weight1}
\end{equation}
\begin{equation}
w_i=
\begin{cases}
\frac{d_k-d_i}{d_k-d_1}\cdot\frac{1}{i},  & d_k\neq d_1 \\
1, & d_k=d_1
\end{cases}
\label{weight2}
\end{equation}

While DW-kNN can improve the performance over V-kNN by introducing the distance information for the decision rule, it loses the relationship information among the kNNs in the neighborhood; thus it may fail to make an effective classification decision under certain circumstance. For example, in Fig. \ref{example2}, the DW-kNN finds ten nearest neighbors from the training samples of two classes; however, while the nearest neighbors from both classes have the same distance information, the DW-kNN will fail to make an effective classification decision in this case.

\subsubsection{Local center-based kNN rules}
Research has investigated the organization of neighborhood information for a more effective classification decision rule resulting in a number of improvements to the kNN rule. The Local Center-based kNN (LC-kNN) rules may address some issues the DW-kNN rules encounter. The LC-kNN generalizes the nearest neighbors from each class as a center prototype for classification \cite{keller1985fuzzy}. The query sample is assigned to the class represented by the nearest center prototype as expressed in \ref{lcknn}, where $d(\cdot)$ is the distance between two samples.
\begin{equation}
\text{LC-kNN: } l=\arg \min_{C}{d(X,\frac{\sum_{i=1}^{k}{I(l_i==C)\cdot{X_i}}}{\sum_{i=1}^{k}{I(l_i==C)}})}
\label{lcknn}
\end{equation}

The basic LC-kNN rule is much less effective for an imbalanced neighborhood than for a balanced neighborhood. The Categorical Average Pattern (CAP) method \cite{hotta2004pattern,mitani2006local} can refine the LC-kNN more effectively by selecting a balanced neighborhood; the \textit{k} nearest neighbors to the query sample are selected from each class to constitute the neighborhood. Suppose the \textit{i}th nearest neighbor in class $C$ is denoted by $X_i^C$, the CAP rule can be expressed as
\begin{equation}
\text{CAP: } l=\arg \min_{C}{d(X,\frac{\sum_{i=1}^{k}{X_i^C}}{k})}.
\label{cap}
\end{equation}

The Local Probabilistic Center (LPC) method proposed by Li et al. \cite{li2008nearest} has demonstrated a refinement to the CAP method by estimating a prior probability that a training sample belongs to its corresponding class. The LPC of each class is estimated based on these prior probabilities to reduce the influence of negative contributing samples. The LPC rule can be expressed as Equation \ref{lpc}, where $p_i$ denotes the probability of $X_i^C$ belonging to class $C$. Li et al. have also designed an estimation of $p_i$ in \cite{li2008nearest}.
\begin{equation}
\text{LPC: } l=\arg \min_{C}{d(X,\frac{\sum_{i=1}^{k}{p_i \cdot X_i^C}}{\sum_{i=1}^{k}{p_i}})}
\label{lpc}
\end{equation}

LC-kNN rules, including CAP and LPC, consider the center information of each class in the neighborhood when making a classification decision. However, the center information may not be sufficient to describe the distribution of the neighborhood. For example, in Fig. \ref{example3} The CAP or LPC selects 5 nearest samples to the query sample for each of the two classes; the local center of Class 2 is closer to the query sample than that of Class 1. However, the sixth nearest sample in Class 1 is closer to the query sample than some of the 5 nearest samples in Class 2, it should therefore be taken into account as a support for classification into Class 1. Thus, it would be more reasonable to classify the query sample to Class 1 due to its more dense distribution around the query sample.

\subsubsection{Support vector machine based kNN rules}
Thanks to the development of support vector machine (SVM) \cite{cortes1995support,chang2011libsvm}, documented research has proposed a hybrid approach using the SVM and kNN (SVM-kNN) to improve the performance of the single classifiers \cite{zhang2006svm,blanzieri2008nearest}. The SVM-kNN is designed to train a local SVM model for each neighborhood to classify the corresponding query sample. In other words, SVM-kNN generates the local decision boundary from the kNNs to classify the particular query sample. We can simply describe the SVM-kNN as
\begin{equation}
l=\mathit{SVM_{N_k(X)}(X)}
\end{equation}
where $\mathit{SVM_{N_k(X)}(\cdot)}$ denotes the SVM model is trained from $N_k(X)$ which is the set of kNNs of $X$.

SVM-kNN can achieve higher classification accuracy than SVM or kNN. However, because the SVM model training is for a particular query sample, it may be time consuming to train a set local SVM models for all query samples, especially for cases where there is a large number of query samples and a large neighborhood size. Furthermore, The SVM can not directly handle multi-class classification problems, therefore, if the neighborhood contains three or more classes, the SVM-kNN must laboriously transform the multi-class classification problem into several binary classification problems \cite{kressel1999pairwise,hsu2002comparison}. In addition, the SVM-kNN only extracts the division information from the neighborhood without regard to the number of kNNs in each class. This may make the SVM-kNN ineffective in some cases. For example, in Fig. \ref{example4} the SVM-kNN generates the optimal decision boundary for the two classes in the neighborhood of the eight nearest neighbors. SVM-kNN will classify the query sample to Class 2 according to the decision boundary, while it may be more reasonable to be classified to Class 1, because there are more and nearer neighbors in Class 1 than in Class 2.

In this article, the proposed LD-kNN method is designed to address the issues identified for the previous kNN rules by taking into account the local distribution information contained in the neighborhood. The decision rule of LD-kNN can comprehensively organize the information contained in the neighborhood and can be effective for most classification problems.

\section{LD-kNN Classification Rule}\label{sec:ldknn}
The kNN algorithms generate the classification decision for a query sample from its neighborhood. To achieve a more representative class for a query sample from its neighborhood, the proposed LD-kNN rule will consider the local distribution information which is a  combination of information used in previous kNN methods.

\subsection{Local distribution}
Suppose a continuous closed region $R$ in the sample space of $X$\footnote{$X$ can be a single variable for univariate distribution or a multivariable for multivariate joint distribution.}, in the $n$ observations of $X$ we have observed \textit{k} samples falling within $R$, we can estimate the prior probability that a point falls in $R$ as $\hat P(R)=k/n$ when $k,n\rightarrow+\infty$. If the region $R$ is small enough, i.e. $V(R)\rightarrow 0$,\footnote{$V(\cdot)$ denotes the size of a region, volume for 3-dimensional cases, area for 2-dimensional cases and hyper-volume for higher dimensional cases.} then the region $R$ can be regarded uniform, thus for each point $X\in R$, we can estimate $f(X)$ by histograms as
\begin{equation}
\mathit{\hat{f}(X)=\frac{P(R)}{V(R)}=\frac{k}{nV(R)}}.
\label{histestimate}
\end{equation}

However, a large $k$ and a small local region $R$ can not coexist in practice. If a larger $k$ is selected to ensure the accuracy of $\hat P(R)=k/n$, the region $R$ can not be always regarded uniform, which may frustrate the histogram estimator. Thus, we propose dealing with this problem by local distribution which is described by Local Probability Density (LPD), versus the global distribution and its probability density function (PDF).

\begin{definition}[Local Probability Density] \label{LPDdef}
In the sample space of variable $X$, given a continuous closed region $R$, for an arbitrary point $X$, the local probability density in $R$ is defined as
\begin{equation}
 f_R(X)=\lim_{\delta(X)\rightarrow X}\frac{P(\delta(X)|R)}{V(\delta(X))}
\end{equation}
where $\delta(X)$ and $V(\delta(X))$ respectively denote the neighborhood of $X$ and the size of the neighborhood, and $P(\delta(X)|R)$ is the conditional probability that a point is in $\delta(X)$ given that it is in $R$.
\end{definition}

In Definition \ref{LPDdef}, the local region $R$ should be predefined for the estimation of $f_R(X)$, if $R$ is defined as the whole sample space, then $P(\delta(X)|R)=P(\delta(X))$, $f_R(X)$ becomes the global PDF. Conversely, if $R$ is small enough, it can be assumed uniform, then $P(\delta(X)|R)=V(\delta(X))/V(R)$ and $f_R(X)=1/V(R)$. Similar to PDF in the whole sample space, LPD describes the local distribution in a certain region. Subsequently, LPD also has the two properties: nonnegativity and unitarity, as Theorem \ref{properties} describes.

\begin{theorem} \label{properties}
If $f_R(X)$ denotes the local probability density of $X$ in a continuous region $R$, then
\begin{equation}
\begin{aligned}
&(I)Nonnegtivity:  \ f_R(X)
\begin{cases}
\geq 0, \quad  for \quad X \in R;  \\
= 0, \quad  for \quad X \notin R
\end{cases} \\
&(II)Unitarity:  \quad \int_{R}{f_R(X)dX}=1.
\end{aligned}
\end{equation}
\end{theorem}

\begin{proof}
The probability $P(\delta(X)|R)$ can not be negative in the sample space, and if $X\notin R$, there must be a sufficiently small neighborhood $\delta(X)$, s.t. $\delta(X)\bigcap R = \Phi$ and then $P(\delta(X) \vert R)=0$, thus according to the definition, $f_R(X)=0$. As shown in the property (I).

For property (II), note that $f_R(X)$ is the differential of $P(X|R)$, the integral of $f_R(X)$ in region $R$ is $P(R|R)=1$.
\end{proof}

We can also generate the relationship between LPD and PDF as follows:
\begin{theorem} \label{pdfest}
In the sample space of variable $X$, if a continuous closed area $R$ has the prior probability $P(R)$, and a point $X\in R$ has the local probability density $f_R(X)$ in a local region $R$, then the global probability density of the point $X$ is
\begin{equation}
f(X)=f_R(X)P(R) \quad  for \quad X\in R.
\label{globalpdf}
\end{equation}
\end{theorem}

\begin{proof}
Due to $X\in R$, then if $\delta(X)\rightarrow X$, we have $\delta(X)\subset R$, and then
\begin{equation}
\begin{aligned}
P(\delta(X))            = P(\delta(X)|R)P(R).
\end{aligned}
\end{equation}

The global probability density can be computed as
\begin{equation}
\begin{aligned}
f(X) &= \lim_{\delta(X)\rightarrow X}{\frac{P(\delta(X))}{V(\delta(X))}} \\
     &= \lim_{\delta(X)\rightarrow X}{\frac{P(\delta(X)|R)P(R)}{V(\delta(X))}}  \\
     &=f_R(X)P(R).
\end{aligned}
\end{equation}
\end{proof}

Theorem \ref{pdfest} provides a method to estimate global probability density from LPD. Due to the locality of LPD, it is supposed much simpler than the global density. Thus we can assume a simple parametric probabilistic model in $R$, and estimate the model parameters from the samples falling in $R$. In the histogram estimation, the area $R$ is assumed uniform, and we can derive $f_R(X)=\frac{1}{V(R)}$ and then the Formula \ref{histestimate}, where the condition $V(R)\rightarrow 0$ is to make the uniform assumption more likely to be correct.

However, the local region $R$ can not be always assumed uniform. If we assume a probabilistic model with the probability density function $f(X;\theta)$ in the local area $R$. We get the estimation of parameter $\hat\theta$ from the points falling in $R$. To ensure the unit measure, i.e. property (II), of LPD, the estimation of LPD should be normalized as
\begin{equation}
\hat f_R(X)=\frac{f(X;\hat\theta)}{\int_{R}{f(X;\hat\theta)dX}}.
\label{lpdest}
\end{equation}

As the global probability can be estimated based on LPD, the kNN methods can turn to the LPD for the distribution information to make an effective classification decision.

\subsection{LD-kNN formulation}
Given a specified sample $X$, by maximum posterior hypothesis, we can generate an effective classification rule by assigning $X$ to the class with the maximum posterior probability conditioned on $X$. In other words, the purpose is looking for the probability that sample $X$ belongs to class $C$, given that we know the attribute description of $X$ \cite{han2006data}. The predicted class of $X$ (denoted as $l$) can be formulated as
\begin{equation}
\begin{aligned}
l &=\arg\max_{C}{P(C|X)} = \arg\max_{C} \frac{P(X|C)P(C)}{P(X)}  \\
&= \arg\max_{C} P(X|C)P(C)
\end{aligned}
\label{maxposter}
\end{equation}

Only $P(X|C)P(C)$ need to be maximized. $P(C)$ is the class prior probability and can be estimated by $N_{C,T}/N_T$, where $N_{C,T}$ is the number of samples in class $C$ in the training set, and $N_T$ denotes the number of all the samples in the training set. As the training set is also constant, the classification problem can be transformed to
\begin{equation}
l=\arg\max_{C}{N_{C,T}P(X|C)}.
\label{eq:max1}
\end{equation}

If the attributes are continuous valued, we can use the local distribution to estimate the related conditional probability density $f(X|C)$ as a substitute of $P(X|C)$. We can select a neighborhood of $X$ $\delta_C(X)$ to estimate $f(X|C)$ through Formula \ref{globalpdf} as
\begin{equation}
\begin{aligned}
f(X|C) &= f_{\delta_C(X)}(X|C) \cdot P(\delta_C(X))  \\
&=f_{\delta_C(X)}(X|C) \cdot N_C/N_{C,T}
\end{aligned}
\end{equation}
where $N_C$ is the number of samples from class $C$ falling in $\delta_C(X)$ in the training set.

Then, the maximization problem in Equation \ref{eq:max1} can be transformed to
\begin{equation}
l  = \arg\max_{C}\{N_C \cdot f_{\delta_C(X)}(X|C)\}.
\label{rule}
\end{equation}

Subsequently, the only thing needed is to estimated the LPD of class $C$ at $X$ (i.e. $f_{\delta_C(X)}(X|C)$) in the neighborhood.

\subsection{Local distribution estimation}
For a query sample $X$, if we have observed $N_C$ samples from class $C$ falling in the neighborhood $\delta_C(X)$ (denoted as $X_i^{C}(i=1,\cdots,N_{C})$), the estimation of LPD is similar to the PDF estimation in the whole sample space except that the local distribution is supposed to be much simpler, and thus we can assume a simple distribution model in $\delta_C(X)$. In our research, we assume a Gaussian model in $\delta_C(X)$ for class $C$, and estimate the corresponding parameters from $\delta_C(X)$. As a contrast, we also use another complicated probabilistic model which is estimated using kernel density estimation (KDE) in the neighborhood.

\subsubsection{Gaussian Model Estimation}
Gaussian Model Estimation (GME) is a parametric method for probability estimation. In GME, we assume that the samples in the neighborhood in each class follow a Gaussian distribution with a mean $\mu$ and a covariance matrix $\Sigma$ defined by Equation \ref{gaussian}, where $d$ is the number of the features.
\begin{equation}
f(X;\mu,\Sigma)=\frac{1}{\sqrt{(2\pi)^d|\Sigma|}}\exp\{-\frac{(X-\mu)^T\Sigma^{-1}(X-\mu)}{2}\}.
\label{gaussian}
\end{equation}

For a certain class $C$, we can take the maximum likelihood estimation (MLE) to estimate the two parameters (the mean $\mu_C$ and the covariance matrix $\Sigma_C$) from the $N_C$ samples falling in $\delta_C(X)$ as
\begin{equation}
\hat{\mu}_{C}=\frac{1}{N_{C}}\sum_{i=1}^{N_{C}}{X_i^{C}}
\label{mean}
\end{equation}
\begin{equation}
\hat{\Sigma}_{C}=\frac{1}{N_{C}}\sum_{i=1}^{N_{C}}{(X_i^{C}-\hat{\mu}_{C})(X_i^{C}-\hat{\mu}_{C})^T}
\label{variance}
\end{equation}

In our approach, to ensure the positive-definiteness of $\Sigma$, we make the naive assumption that all the attributes on each class do not correlate with one another in a local region; that is, the covariance matrix ($\Sigma$) would be a diagonal matrix. Then the MLE of the covariance matrix $\Sigma_{C}$ would be as
\begin{equation}
\hat{\Sigma}_{C}=diag(\frac{1}{N_{C}}\sum_{i=1}^{N_{C}}{(X_i^{C}-\hat{\mu}_{C})(X_i^{C}-\hat{\mu}_{C})^T})
\label{var}
\end{equation}
where $diag(\cdot)$ converts a square matrix to a diagonal matrix with the same diagonal elements.

The Gaussian model is a global model that integrates to unity in the whole sample space, while the LPD should integrate to unity in the local region. Thus, as described in Formula \ref{lpdest}, the Gaussian model in $\delta_C(X)$ should be normalized, we then get the LPD for $X$ in $\delta_C(X)$ for class $C$ as
\begin{equation}
f_{\delta_C(X)}(X|C)=\frac{f(X;\hat\mu_C,\hat\Sigma_C)}{\int_{\delta_C(X)}{f(X;\hat\mu_C,\hat\Sigma_C)dX}}.
\label{GMELPD}
\end{equation}

\subsubsection{Kernel Density Estimation}
We have also conducted KDE in the neighborhood for LPD estimation, where less rigid assumptions have been made about the prior distribution of the observed data. To estimate the LPD $f_{\delta_C(X)}(X|C)$ for a certain class $C$ in the neighborhood, we use the common KDE with identical Gaussian kernels and an assigned bandwidth vector $H_C$, expressed as
\begin{equation}
\hat{f}(X|C)=\frac{1}{N_{C} \cdot prod(H_C)}\sum_{i=1}^{N_{C}}{K((X_i^{C}-X)./H_C)}
\label{kde}
\end{equation}
where the operator $./$ denotes the right division between the corresponding elements in two equal-sized matrices or vectors, $prod(\cdot)$ returns the product of the all elements in a vector, and $K(\cdot)$ is the kernel function. In this article, we use the Gaussian kernel defined as
\begin{equation}
K(X)=\frac{1}{(2\pi)^{d/2}}\exp\{-\frac{X^T \cdot X}{2}\}.
\label{kernel}
\end{equation}

The estimated $\hat{f}(X|C)$ also need to be normalized as
\begin{equation}
f_{\delta_C(X)}(X|C)=\frac{\hat{f}(X|C)}{\int_{\delta_C(X)}{\hat{f}(X|C)dX}}.
\label{KDELPD}
\end{equation}

We selected the bandwidth $h_C$ for a certain class $C$ in a dimension according to Silverman's rule of thumb \cite{silverman1986density,turlach1993bandwidth} by
\begin{equation}
\hat h_C=(\frac{4\hat{\sigma}^5}{3N_C})^{1/5}\approx1.06\hat{\sigma}_C N_C^{-1/5}
\label{bandwidth}
\end{equation}
where $\hat h_C$ is the estimated optimal bandwidth, and $\hat{\sigma}_C$ is the estimated standard deviation of class $C$ in $\delta_C(X)$. The $h_C$ for each dimension is computed independently.

\subsection{Generalized LD-kNN rules}
In kNN methods, $\delta_C(X)$ in Formula \ref{rule} is constant for all classes and is selected as a neighborhood that contains $k$ samples in the training set. Thus, for a query sample $X$ and a training set $T$, given the parameter $k$ and a certain distance metric $d(\cdot)$, a generalized kNN method can execute the following steps to evaluate which class the query sample $X$ should belong to.

\textbf{Step 1:} For the query samples $X$, find its $k$ nearest neighbors as a neighborhood set $N_k(X)$ from the training set $T$ according to the distance function, i.e.,
\begin{equation}
N_k(X)=\{x|d(x,X)\leq d(X_k,X), x \in T\}
\end{equation}
where $X_k$ is the \textit{k}th nearest sample to $X$ in $T$.

\textbf{Step 2:} Divide the neighborhood set $N_k(X)$ into clusters according to the class label, such that samples in the same cluster have the same class label while samples from different clusters have different class labels. If the \textit{j}th cluster labeled $C_j$ is denoted by $X^{C_j}$, then
\begin{equation}
X^{C_j}=\{x|l(x)=C_j,x \in N_k(X)\}
\end{equation}
where $l(x)$ denotes the observed class label of $x$, and the division is expressed by
\begin{equation}
X^{C_i}\bigcap X^{C_j} = \Phi \quad  for \quad i\neq j;
\end{equation}
\begin{equation}
N_k(X)=\bigcup_{j=1}^{N}{X^{C_j}}
\end{equation}
\begin{equation}
k=|N_k(X)|=\sum_{j=1}^{N}{|X^{C_j}|}=\sum_{j=1}^{N}{N_{C_j}}
\end{equation}
where $N$ is the number of all classes in $N_k(X)$ and $N_{C_j}$ is the number of objects in the \textit{j}th cluster $X^{C_j}$.

\textbf{Step 3:} For each cluster $X^{C_j}$, estimate $f_{\delta(X)}(X|C)$ by GME or KDE from the corresponding Formulae \ref{GMELPD} or \ref{KDELPD}.

\textbf{Step 4:} The query sample $X$ is classified into the most probable class that has the maximum $N_{C} \cdot f_{\delta(X)}(X|C)$ as expressed by Equation \ref{rule}.

According to the aforementioned processes, an LD-kNN rule assigns the query sample to the class having a maximum posterior probability which is calculated according to the Bayesian theorem based on the local distribution.

\subsection{Analysis}
Formula \ref{rule} can be regarded as a generalized classification form; different classification rules can be generated through selecting different neighborhoods $\delta_C(X)$ for class $C$ or different local distribution estimation methods. From Formula \ref{rule}, the classification decision of LD-kNN has two components, $N_C$ and $f_{\delta_C(X)}(X|C)$, which need to be estimated from the $N_k(X)$. $N_C$ represents quantity information and $f_{\delta_C(X)}(X|C)$ describes the local distribution information in the neighborhood.

\subsubsection{LD-kNN and V-kNN}
The traditional V-kNN rules classify the query sample using only the number of nearest neighbors for each class in the $N_k(X)$ (i.e. $N_{C}$ for the class $C$). Comparing the classification rules of V-kNN and LD-kNN expressed respectively in Formulae \ref{vknn} and \ref{rule}, the LD-kNN rule additionally takes into account the LPD at the query sample (i.e. $f_{\delta_C(X)}(X|C)$) for each class. In other words, the V-kNN rule assumes a constant LPD $f_{\delta_C(X)}(X|C)$ for all classes, i.e., it assumes a uniform distribution in a constant neighborhood for all classes. However, for different classes, the local distributions are not always identical and may play a significant role for classification; the proposed LD-kNN rules take the differences into account through a local probabilistic model.

As exemplified in Fig. \ref{example1}, while the number of samples from Class 1 are less than that from Class 2 in the neighborhood of the query sample, the local distribution of Class 1 estimated from the nearer samples can be more dense around the query sample and can achieve a greater posterior probability than Class 2.

\subsubsection{LD-kNN and DW-kNN}
DW-kNN rules assume that the voting of kNNs should be based on a distance-related weight, because the query sample is more likely to belong to the same class with its nearer neighbors. It is an effective refinement to the V-kNN rule. The DW-kNN can be regarded as a special case of LD-kNN with KDE. Combining the Formulae \ref{rule}, \ref{kde} and \ref{KDELPD}, we can get the rule as
\begin{equation}
     l    = \arg\max_{C}\{\sum_{i=1}^{N_{C}}{w_i^C}\}
\label{kderule}
\end{equation}
where
\begin{equation}
w_i^C=\frac{K((X_i-X)./H_C)}{prod(H_C)\int_{\delta_C(X)}{\hat{f}(X|C)dX}}
\end{equation}

Thus, the LD-kNN can be clearly viewed as a weighted kNN rule with a weight $w_i^C$ assigned to $X_i$ in class $C$; If we further assign $H_C=[1,\cdots,1]$ and omit the differences of $\int_{\delta_C(X)}{\hat{f}(X|C)dX}$ among classes, the weight is related to the kernel function $K(X_i-X)$. Though the kernel function $K(\cdot)$ has the constraint condition that it should ingrate to 1 in the whole space, it does not influence the weight ratio for each sample by a normalization process.

In this sense, LD-kNN with KDE takes into account more information (e.g. the band widths related to classes $H_C$) than DW-kNN. In addition, the local normalization for unity measure as Formula \ref{KDELPD} is a reasonable process, though the differences of $\int_{\delta_C(X)}{\hat{f}(X|C)dX}$ among classes are usually tiny and can be omitted.

As shown in Fig. \ref{example2}, while samples from Class 1 have the same number and the same distances to the query sample as compared to the samples from Class 2, the distribution for the two classes is quite different. According to the local distribution information, LD-kNN can make an effective classification decision through the posterior probability.

\subsubsection{LD-kNN and LC-kNN}
LC-kNN generates the classification rule only based on the local center, which is a part of the local distribution in GME. It can also be derived from LD-kNN with GME with a constant covariance matrix for all classes. If the neighborhood for each class $C$ is selected so that it has a constant number of samples in the training set, i.e., $N_C$ is constant for all classes, the classification rule in Formula \ref{rule} can be transformed as Equation \ref{ld2lc} with a constant covariance matrix $\Sigma$ for GME.
\begin{equation}
\begin{aligned}
     l &= \arg\max_{C}\{f(X;\hat\mu_C,\Sigma)\}   \\
       &= \arg\min_{C}(X-\hat\mu_C)^T\Sigma^{-1}(X-\hat\mu_C)  \\
       &= \arg\min_{C} d_m^2(X,\hat\mu_C)
 \end{aligned}
 \label{ld2lc}
\end{equation}
where $d_m(\cdot)$ is the Mahalanobis distance between two samples. As can be seen, through assuming an equal covariance matrix $\Sigma$ for all classes, LD-kNN with GME can shrink to LC-kNN. If the covariance matrix $\Sigma$ is further assumed the identity matrix, the Mahalanobis distance reduces to the Euclidean distance, then the LD-kNN can be reduced to the CAP or LPC, depending on the estimation of local center.

Taking into account the differences among the local variabilities of all classes, LD-kNN can be an improved version of LC-kNN. For example, in Fig. \ref{example3}, while the local center of Class 2 is closer to the query sample than that of Class 1, there is a greater concentration of samples in Class 1 than that in Class 2 around the query sample and the LD-kNN may calculate a greater posterior probability for Class 1.

\subsubsection{LD-kNN and Bayesian rules}
The Bayesian rules estimate the conditional probability density $f(X|C)$ corresponding to $P(X|C)$ on the whole dataset by a Gaussian model or KDE. However, for a certain sample $X$, the estimation of $f(X|C)$ on the whole dataset usually affected by the noises on the whole dataset, while LD-kNN rules estimate $f(X|C)$ through the local distribution in $\delta_k(X)$ that the noises outside this area can not affect the estimation of $f(X|C)$. And moreover, with the knowledge that near neighbors can usually represent the property of a sample better than the more distant samples, estimating the local distribution in the neighborhood of the query sample to predict its class label is quite feasible.

In fact, LD-kNN can be considered as a compromise between the 1-nearest-neighbor rule and the Bayesian rule. The parameter \textit{k} describes the locality of LD-kNN; when parameter \textit{k} is close to 1, LD-kNN approaches the nearest neighbor rule; if \textit{k} increases to the size of the dataset, then the local area is extended to the whole dataset, and LD-kNN becomes a Bayesian classifier. In this sense, through tuning the parameter \textit{k}, LD-kNN may combine the advantages of the two classifiers and become a more effective method for classification.

\subsubsection{LD-kNN and SVM-kNN}
The SVM-kNN method generates the classification rule based on the optimal decision boundary which can be viewed as the division information extracted from the neighborhood. Both SVM-kNN and LD-kNN can generate a local classification model in a neighborhood, SVM-kNN generates a local SVM model while LD-kNN generates a local probabilistic model. Different from the local probabilistic model in LD-kNN, the local SVM model in SVM-kNN seldom considers the quantity information corresponding to the prior probability of a class in the neighborhood; also, it does not care how the other samples are distributed provided they will not become support vectors and influence the decision boundary. Thus, due to the lack of distribution information, SVM-kNN may be less effective than LD-kNN in the case depicted in Fig. \ref{example4} as the LD-kNN should achieve a greater posterior probability for Class 1, due to its greater prior probability in the neighborhood.

\subsection{Computational complexity} \label{complexity}
In the classification stage, the kNN methods generally: (a) initially compute the distances between the query sample and all the training samples, (b) identify the kNNs based on the distances, (c) organize the information contained in the neighborhood, and (d) generate the classification decision rules. If $m$, $d$ and $k$ respectively denote the number of the training samples, the number of features, and the number of the nearest neighbors. In step (a), if the Euclidean distance is employed, we require $O(md)$ additions and multiplications. In step (b), we require $O(km)$ comparisons for the kNNs. In step (d), these kNN rules all search for the maximum or minimum value in a queue with $c$ values corresponding to the decision values of $c$ classes; this step only requires $O(c)$ comparisons, which can be omitted compared with the other steps.

The LD-kNN method differs from other kNN methods in organizing the information in the neighborhood in step (c). For the step (c), different kNN rules will have different computational cost. If there are $c$ classes in the neighborhood, then the computational complexity of some related kNN rules in step (c) can be easily analyzed and shown in Table \ref{tab:complexity}. From Table \ref{tab:complexity}, the LD-kNN may have somewhat more time complexity than other kNN rules except for LPC. However, in practical problems, we usually have $c< k < m$, and $c \ll m$, thus, compared to the computational complexity in step (a) and (b) the computational cost in step (c) can be omitted. Then, the total computational complexity for LD-kNN is $O(md)$ in terms of addition and multiplication, and $O(km)$ in terms of comparison which is equal to that of the V-kNN.

SVM-kNN has the same computation in steps (a) and (b), while in steps (c) and (d) it trains an SVM model based on the kNNs and predicts the label of the query sample according to the trained model. For a binary SVM, the computational complexity is highly related to the number of support vectors $N_S$. As Burges et al. \cite{burges1998tutorial} described, the computational complexity of training a local SVM is $O(N_S^2+ N_Skd)$ at the best case and $O(k^2d)$ or $O(N_S^3 + N_Skd)$ at the worst case. Since the asymptotical number of support vectors grows linearly with the number of training samples, the computational cost grows between $O(k^2d)$ and $O(k^3+k^2d)$ \cite{bottou2007support}. In the testing phase, if we use the Radial Basis Function (RBF) kernel or linear kernel, the computational complexity is $O(N_Sd)$. Thus, combining the computation in steps (a) and (b), the overall complexity of SVM-kNN can be at least $O(md+km+k^2d)$.

\begin{table}
  \centering
  \caption{The computational complexity for some related kNN rules in the step of neighborhood information organization}
    \begin{tabular}{l|cccc}
    \toprule
    kNN rules & Additions & Multiplications & Comparisons \\
    \midrule
    V-kNN  & O(k)  &       & O(kc) \\
    LD-kNN(GME) & O(kd) & O(kd) & O(kc) \\
    LD-kNN(KDE)  & O(kd) & O(kd) & O(kc) \\
    DW-kNN$^{1}$  & O(k)  & O(k)  & O(kc) \\
    CAP$^2$   & O(kd) & O(dc) & O(mc)\\
    LPC$^3$    & O(kd) & O(kd) & O(mc)  \\
    \bottomrule

    \end{tabular}%
  \label{tab:complexity}%
  \flushleft

  $^1$ Here, we refer to the weighting function as Equations \ref{weight1} and \ref{weight2} which have equal computational complexities.

  $^2$ For CAP and LPC, the dataset should be grouped before the neighborhood information organization step; the computational  complexity is $O(mc)$.

  $^3$ LPC should have an additional process of estimating the related probability for each training sample, which can be achieved offline. This process will take the major computational cost $O(k_0m^2)$ if using the method described in \cite{li2008nearest}.
\end{table}%

\section{Experiments}\label{sec:experiments}
To study the performance characteristics of the proposed LD-kNN method we have conducted four sets of experiments using synthetic and real datasets. Experiment I studies the influence of neighborhood size on the performance of kNN-based classifiers; Experiment II studies the scalability of dimension for the proposed method; Experiment III studies the efficiency of LD-kNN rules; Experiment IV studies the classification performance of LD-kNN rules for real classification problems.

\subsection{Experimental datasets}
To evaluate the LD-kNN approach, the datasets used in our experiments contain 27 real datasets and four types of synthetic datasets.

\subsubsection{The synthetic datasets}
Each of the four types of synthetic datasets (denoted as T1, T2, T3, T4) is used for a two-class classification problem. There are at least three advantages in using synthetic datasets \cite{mitani2006local}: (1) the size of training and test sets can be controlled; (2) the dimension of the dataset and the correlation of the attributes can be controlled; and (3) for any two-class problem, the true Bayes error can be obtained \cite{fukunaga1969calculation}. If a \textit{p}-dimensional sample is denoted as $(x_1,\cdots,x_{p-1},y)$, the four types of synthetic datasets are described as follows.

In T1 datasets, the data samples of the two classes are in a uniform distribution in two adjacent regions divided by a nonlinear boundary: $y=\frac{1}{p-1}\sum_{i=1}^{p-1}{\sin(x_i)}$. The uniform region is a hyper-rectangle area: $0 \leq {x_i} \leq 2\pi, -2 \leq y \leq 2, i=1,\cdots,p-1$.

In the T2, T3 and T4 datasets the data samples of the two classes are from two multivariate Gaussian distributions $N(\mu_1,\Sigma_1)$ and $N(\mu_2,\Sigma_2)$.

In the T2 dataset, the two Gaussian distributions have the same diagonal covariance matrix but different mean vectors, i.e.,
\begin{equation}
\begin{split}
&\mu_1=[\textbf{0}_{p-1},-1], \mu_2=[\textbf{0}_{p-1},1]  \\
&\Sigma_1=\Sigma_2=I_p
\end{split}
\end{equation}
where $I_p$ denotes the \textit{p}-dimensional identity matrix, $\textbf{0}_{p}$ denotes a \textit{p}-dimensional zero vector.

In the T3 dataset, the two Gaussian distributions have the same mean vector but different diagonal covariance matrixes,
\begin{equation}
\begin{split}
&\mu_1=\mu_2=\textbf{0}_{p}\\
&\Sigma_1=I_p,\Sigma_2=4I_p
\end{split}
\end{equation}

In the T4 dataset, the two Gaussian distributions have different mean vectors and non-diagonal covariance matrixes, i.e.,
\begin{equation}
\begin{split}
&\mu_1=[\textbf{0}_{p-1},-1], \mu_2=[\textbf{0}_{p-1},1]\\
&\Sigma_1=\textbf{1}_p+I_p,\Sigma_2=\textbf{1}_p+3I_p
\end{split}
\end{equation}
where $\textbf{1}_p$ denotes a \textit{p}-dimensional square matrix with all components 1, $\Sigma_1$ and $\Sigma_2$ are non-diagonal matrixes with the corresponding main diagonal components 2 and 4, and with the remaining components equal to 1.

\subsubsection{The real datasets}
The real datasets are selected from the well-known UCI-Irvine repository of machine learning datasets \cite{Bache+Lichman:2013}. Because the estimation of probability density is only for numerical attributes, the attributes of the selected datasets are all numerical. These datasets cover a wide area of applications including life, computer, physical and business domains. The datasets include 13 two-class problems and 14 multi-class problems. Table \ref{tab:datainfo} summarizes the relevant information for these datasets.

\begin{table}
  \centering
  \caption{Information about the real datasets}
    \begin{tabular}{lcccl}
    \toprule
    Datasets & \#Instances    & \#Attributes    & \#Classes    & Area \\
    \midrule
    Blood & 748   & 4     & 2     & Business \\
    BupaLiver & 345   & 6     & 2     & Life \\
    Cardio1 & 2126  & 21    & 10    & Life \\
    Cardio2 & 2126  & 21    & 3     & Life \\
    Climate & 540   & 18    & 2     & Physical \\
    Dermatology & 366   & 33    & 6     & Life \\
    Glass & 214   & 9     & 7     & Physical \\
    Haberman & 306   & 3     & 2     & Life \\
    Heart & 270   & 13    & 2     & Life \\
    ILPD  & 583   & 10    & 2     & Life \\
    Image & 2310  & 19    & 7     & Computer \\
    Iris  & 150   & 4     & 3     & Life \\
    Leaf  & 340   & 14    & 30    & Computer \\
    Pageblock & 5473  & 10    & 5     & Computer \\
    Parkinsons & 195   & 22    & 2     & Life \\
    Seeds & 210   & 7     & 3     & Life \\
    Sonar & 208   & 60    & 2     & Physical \\
    Spambase & 4601  & 57    & 2     & Computer \\
    Spectf & 267   & 44    & 2     & Life \\
    Vehicle & 846   & 18    & 4     & Physical \\
    Vertebral1 & 310   & 6     & 2     & Life \\
    Vertebral2 & 310   & 6     & 3     & Life \\
    WBC   & 683   & 9     & 2     & Life \\
    WDBC  & 569   & 30    & 2     & Life \\
    Wine  & 178   & 13    & 3     & Physical \\
    WinequalityR & 1599  & 11    & 6     & Business \\
    WinequalityW & 4898  & 11    & 7     & Business \\
    \bottomrule
    \end{tabular}%
    \flushleft

  \label{tab:datainfo}%
\end{table}%

\subsection{Experimental settings}
The shared experimental settings for all the four experiments are described as follows. The unique settings for each experiment are introduced in the corresponding subsection.

\subsubsection{The distance function}
While different classification problems may utilize different distance functions, in our experiments we use Euclidean distance to measure the distance between two samples.

\subsubsection{The normalization}
To prevent attributes with an initially large range from inducing bias by out-weighing attributes with initially smaller ranges, in our experiments, all the datasets are normalized by a z-score normalization that linearly transforms each of the numeric attributes of a dataset with mean 0 and standard deviation 1.

\subsubsection{The Parameter}
The parameter \textit{k} in kNN-based classifiers indicates the number of nearest neighbors. In our experiments, we use the average number of nearest neighbors per class (denoted by $kpc$) as the parameter, i.e., we totally search $kpc*N_C$ nearest neighbors, where $N_C$ is the number of classes in the corresponding dataset.

\subsubsection{The performance evaluation}
The performance of a classifier usually contains two aspects, effectiveness and efficiency. The effectiveness is related to the data distribution in a dataset; and the efficiency is usually related to the size of a dataset. The misclassification rate (MR) and F1-score (F1) are employed to assess the effectiveness of a classifier; a smaller MR or a greater F1 denotes better effectiveness. In Experiment I and II, due to the simplicity of synthetic datasets, we only use MR to evaluate the effectiveness; while both MR and F1 are calculated in the real datasets in Experiment IV. The efficiency is described by the time consumed in a classification task in Experiment III.

For Experiments I, II and IV, to express the generalization capacity, the training samples and the test samples should be independent. In our research we use stratified 5-fold cross validation to estimate the effectiveness indicator (MR or F1) of a classifier on each dataset. In stratified 5-fold cross validation, the data are randomly stratified into 5 folds. Among the 5 folds, one is used as test set, and the remaining folds are used as the training set. We perform the classification process 5 times, each time using a different test set and the corresponding training set. To avoid bias, we repeat the cross validation process on each dataset ten times and calculate the average MR (AMR) or average F1 (AF1) to evaluate the effectiveness of a classifier.

\subsection{Experiment I }
The purpose of Experiment I is to investigate the influence of the parameter \textit{kpc} on effectiveness of kNN-based classifiers. This investigation is designed to guide us in the selection of the optimal parameter \textit{kpc} for classification. The 4 types of synthetic datasets with various dimensions and several real datasets are employed in this experiment. For each synthetic dataset there are 500 samples for each of the two classes. Additionally, we set the dimension of synthetic datasets in \{2, 5, 10\} to investigate the influence of dimension on \textit{kpc}'s selection.

In this experiment, two LD-kNN rules are implemented using GME and KDE respectively. Some other kNN-based rules are also employed as controls, including V-kNN, DW1-kNN, DW2-kNN, CAP, LPC, SVM-kNN(RBF), and SVM-kNN(Poly). \footnote{In the rest of this article, LD-kNN(GME) and LD-kNN(KDE) denote the LD-kNN rules with the local distribution estimated using GME and KDE respectively; DW1-kNN and DW2-kNN denote the distance weighted kNN rules with the weight function Eq. \ref{weight1} and Eq. \ref{weight2} respectively; SVM-kNN(RBF) and SVM-kNN(Poly) denote the SVM-kNN rules with the respective kernel RBF and Polynomial for SVM.}

The AMR of the kNN-based classifiers on synthetic datasets varied with respect to parameter \textit{kpc} are depicted in Fig. \ref{fig_kpc2} where each subfigure represents the performances of the kNN-based classifiers on a specified type of dataset with a specified dimension. Each row in in Fig. \ref{fig_kpc2} denotes a type of synthetic dataset, T1 to T4 from top to bottom. A column in Fig. \ref{fig_kpc2} denotes the dimension of the synthetic dataset; the columns from left to right represent the dimension p=2, p=5 and p=10 respectively.

\begin{figure*}[!t]
\centering
\subfloat{\includegraphics[width=0.9\textwidth]{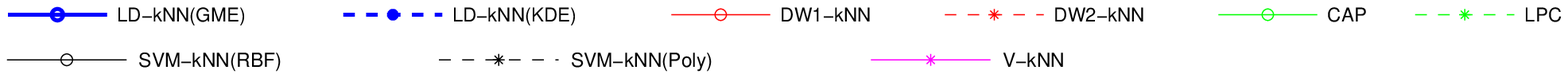}}

\addtocounter{subfigure}{-1}

\subfloat[T1 dataset (p=2)]{\includegraphics[width=0.33\textwidth]{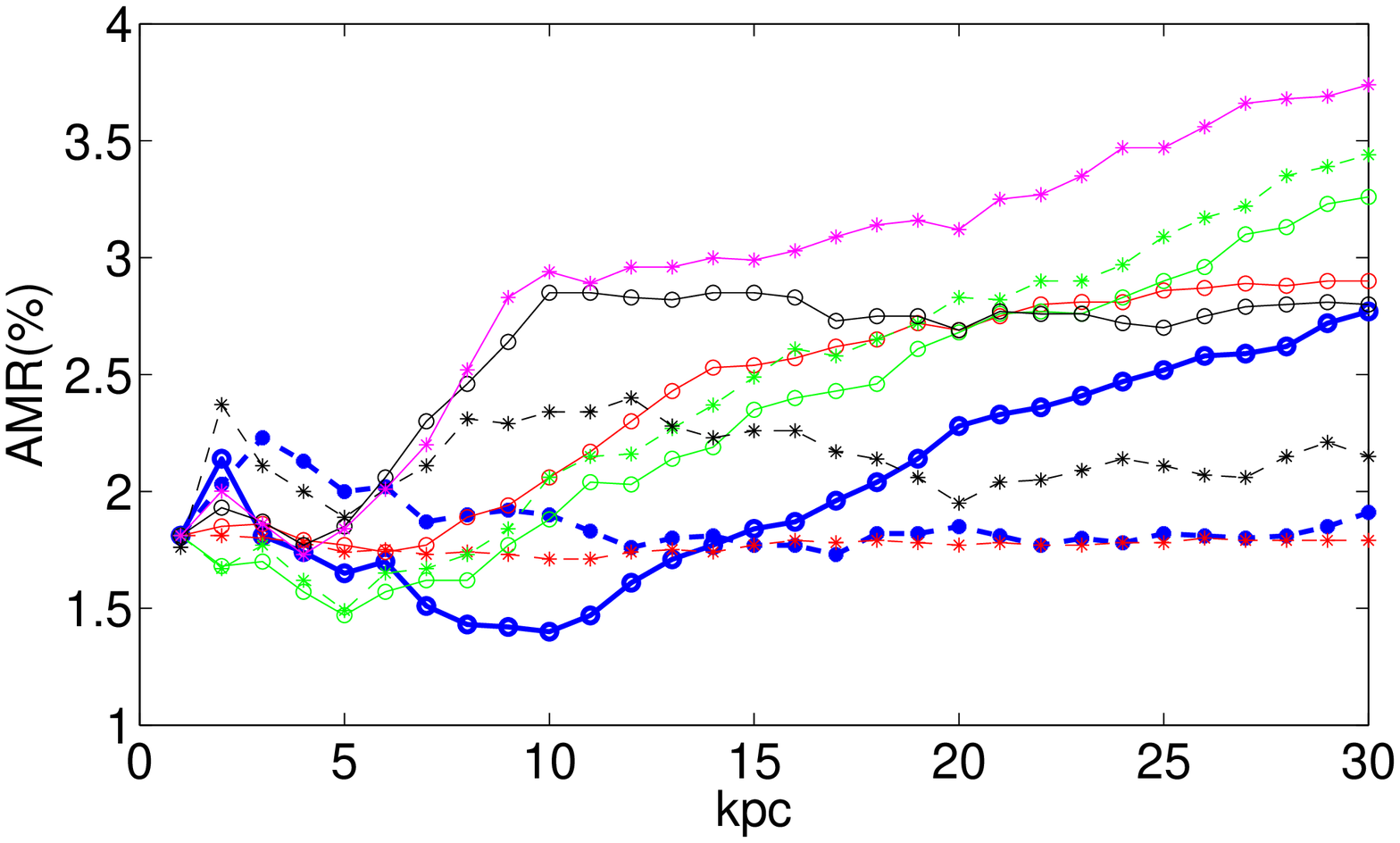} \label{T2_2}}
\subfloat[T1 dataset (p=5)]{\includegraphics[width=0.33\textwidth]{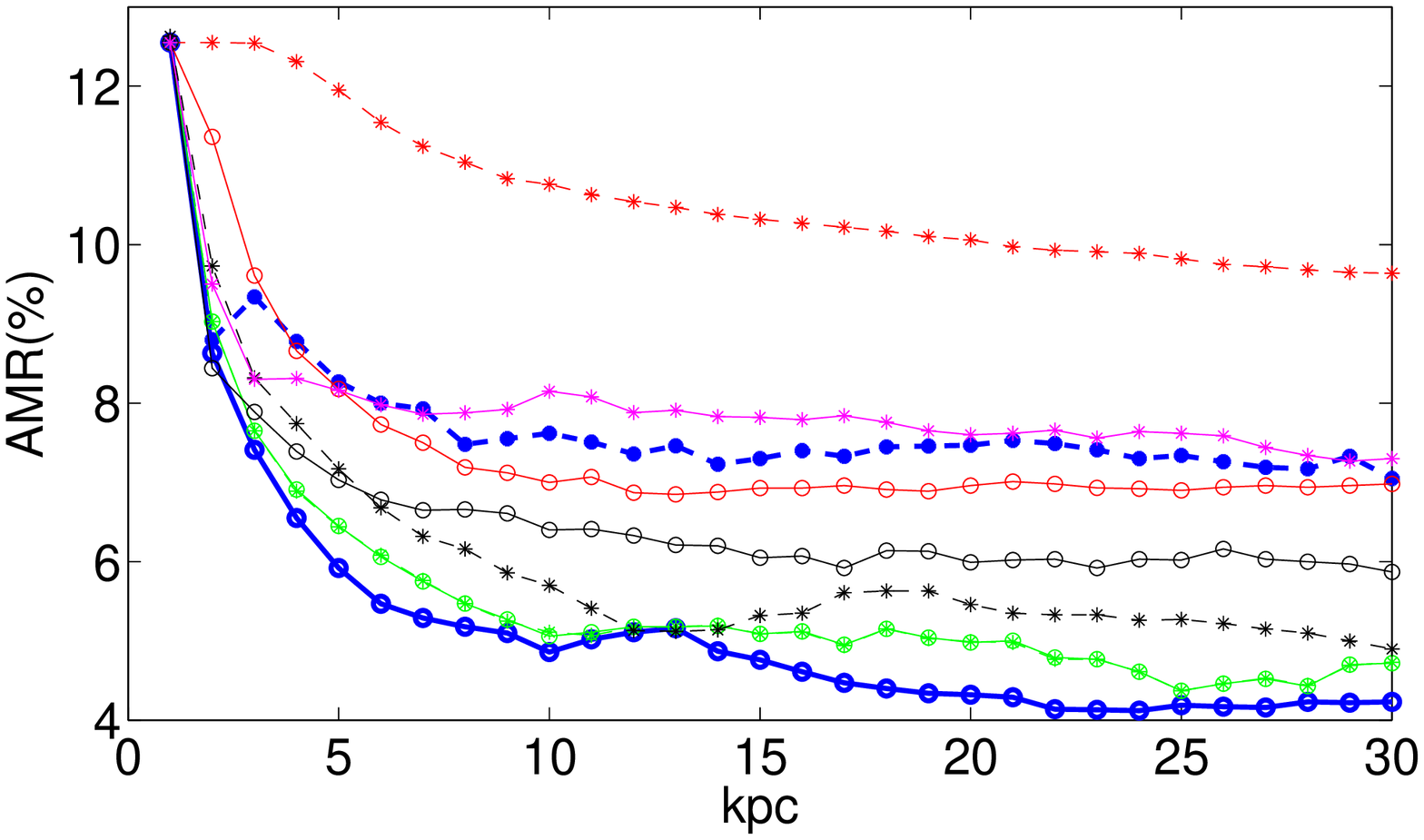} \label{T2_5}}
\subfloat[T1 dataset (p=10)]{\includegraphics[width=0.33\textwidth]{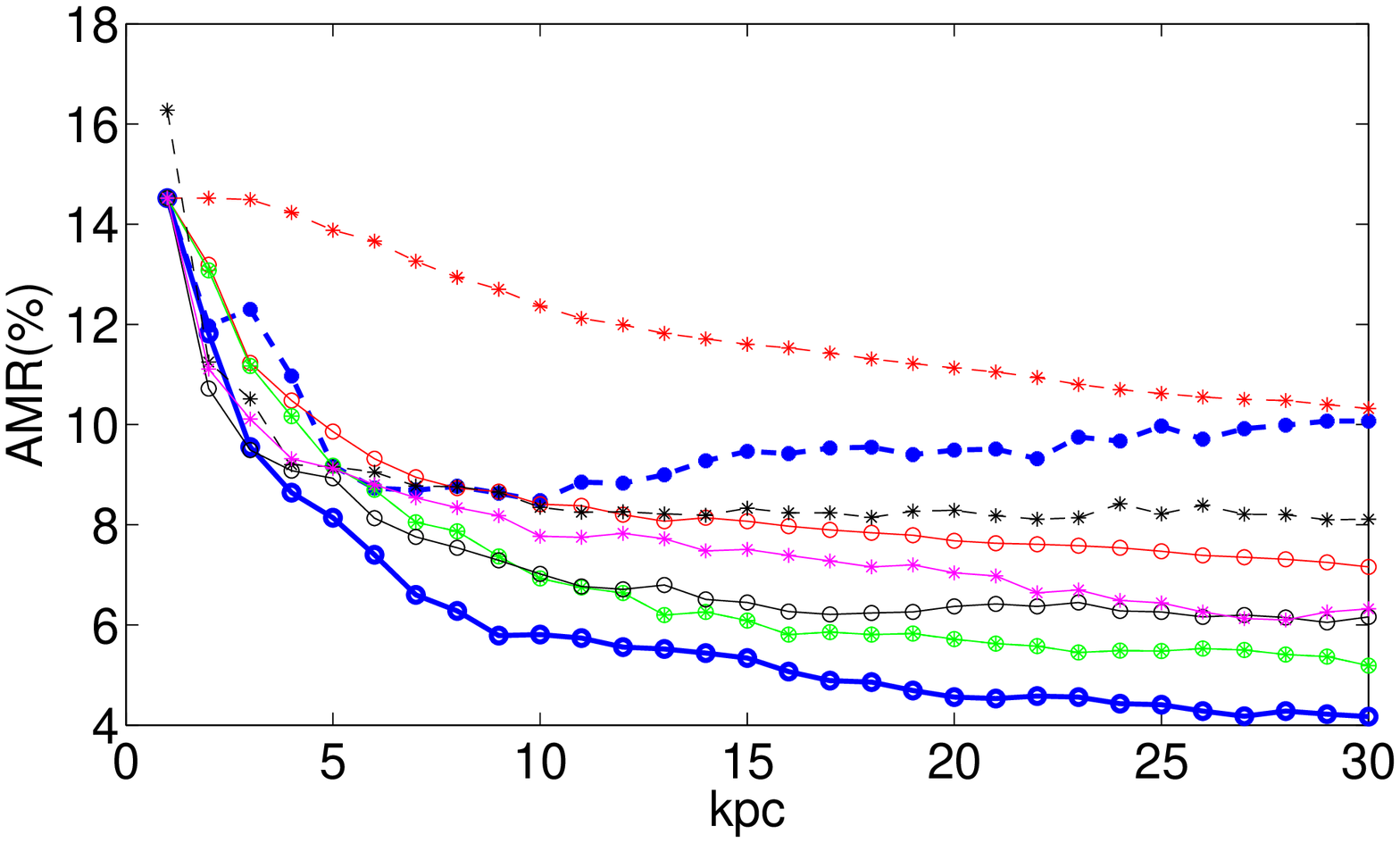} \label{T2_10}}

\subfloat[T2 dataset (p=2)]{\includegraphics[width=0.33\textwidth]{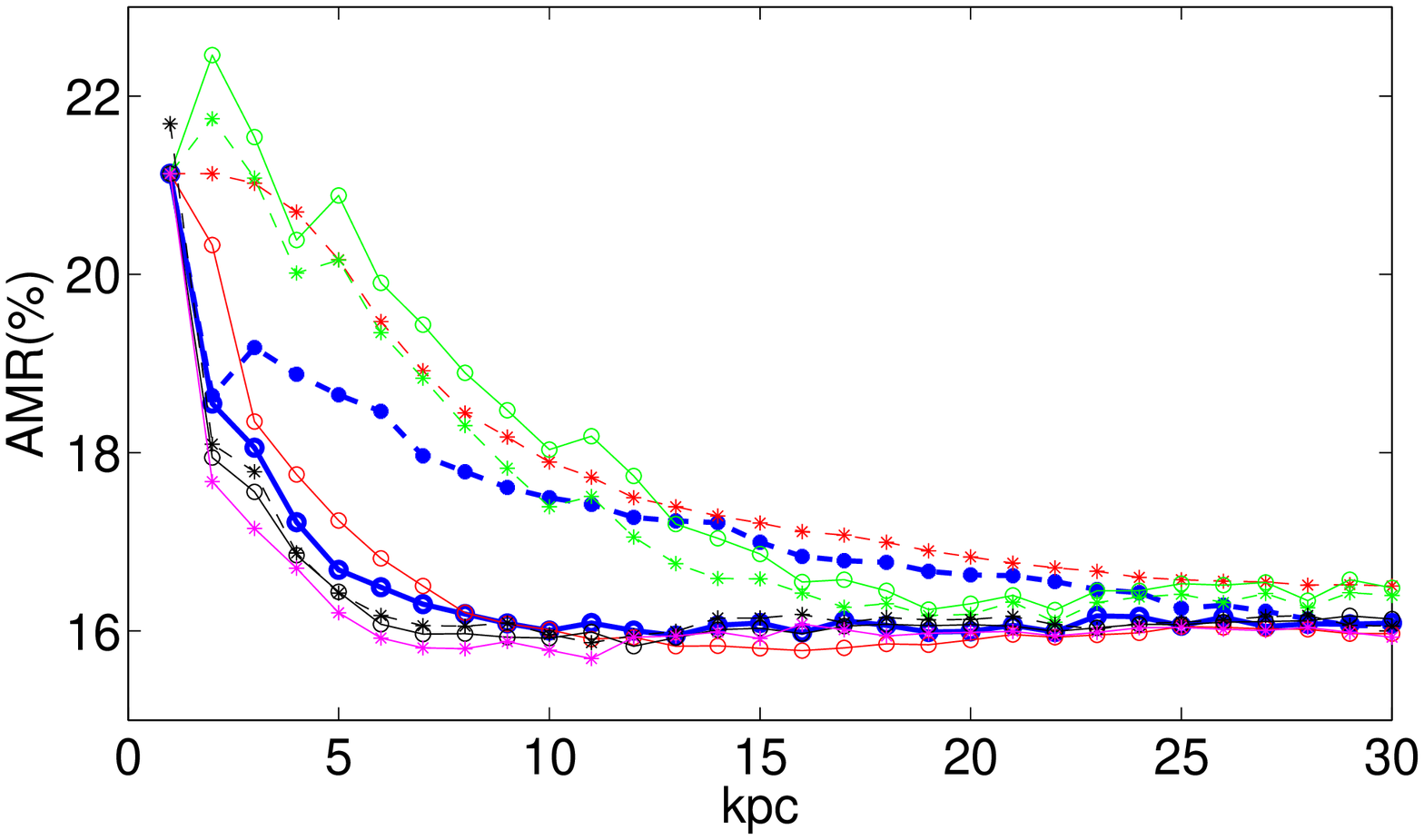}}
\subfloat[T2 dataset (p=5)]{\includegraphics[width=0.33\textwidth]{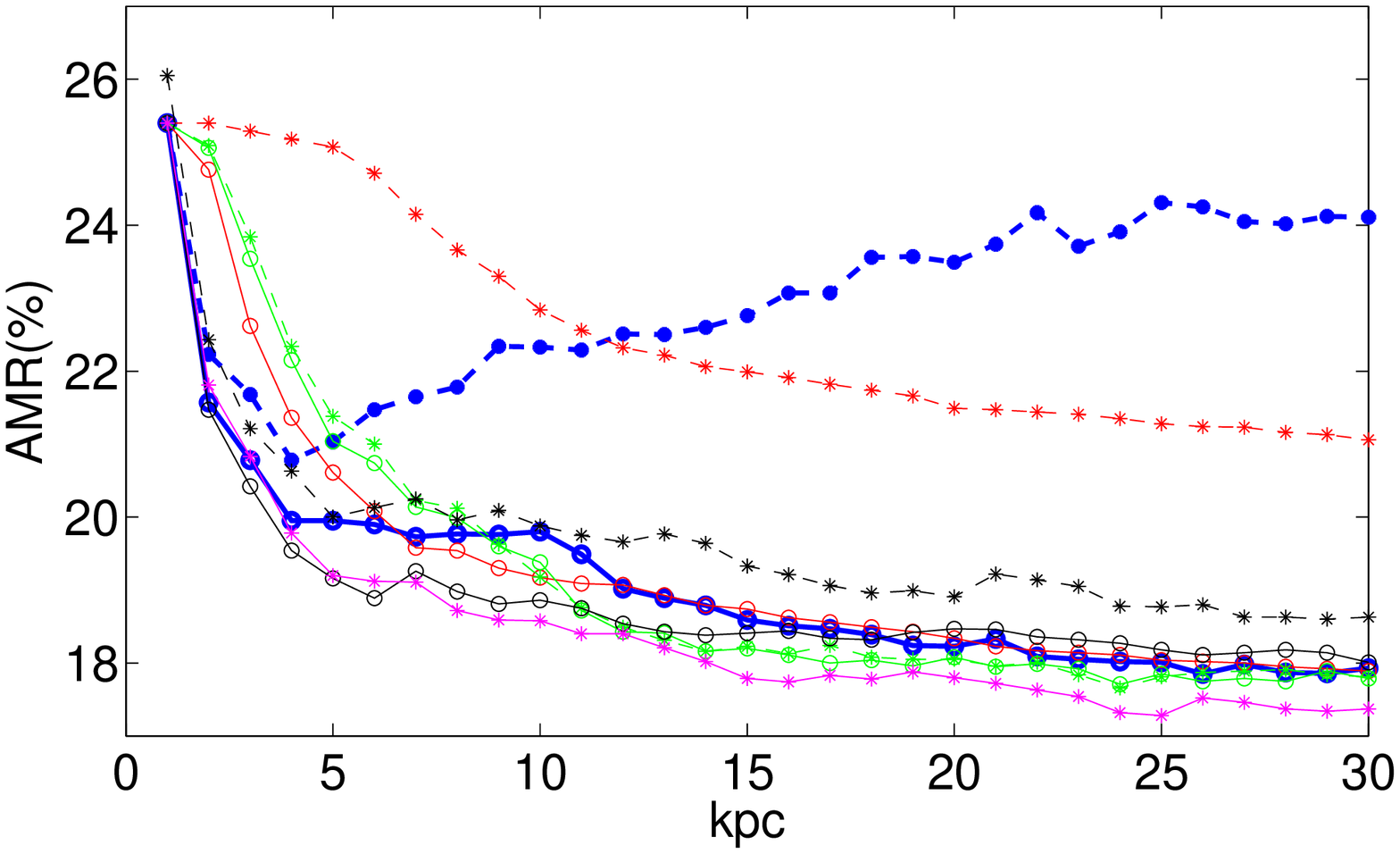}}
\subfloat[T2 dataset (p=10)]{\includegraphics[width=0.33\textwidth]{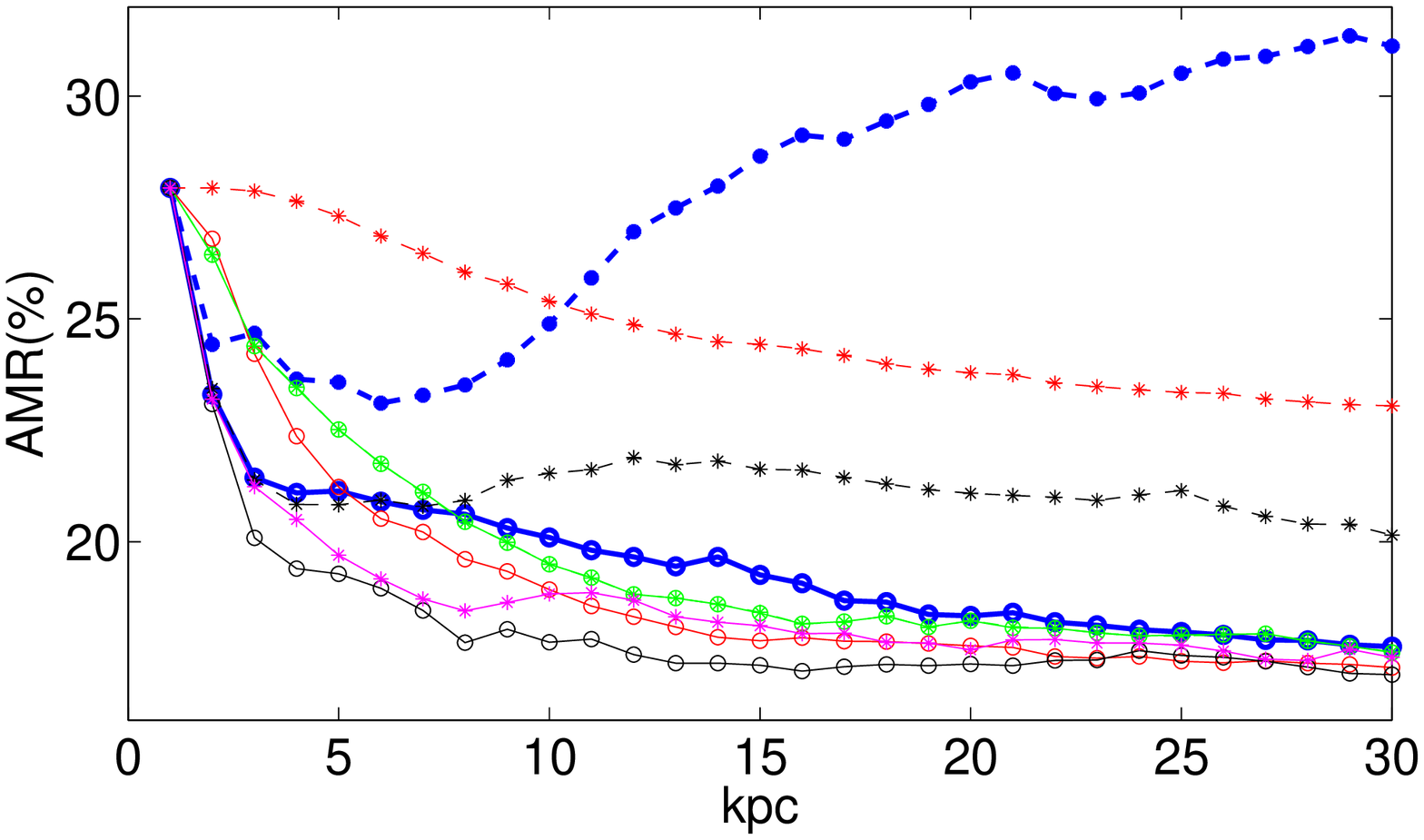}}

\subfloat[T3 dataset (p=2)]{\includegraphics[width=0.33\textwidth]{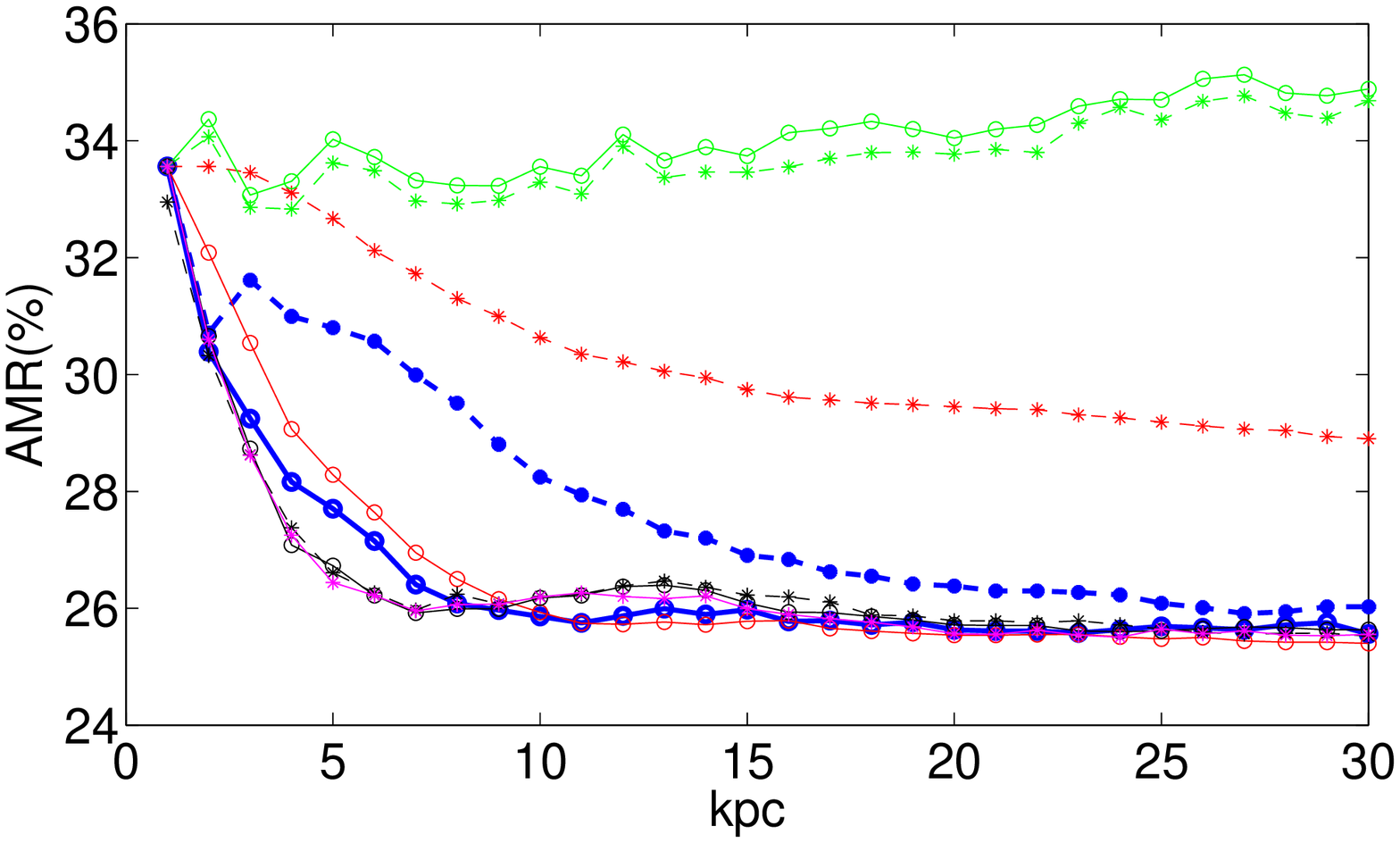}}
\subfloat[T3 dataset (p=5)]{\includegraphics[width=0.33\textwidth]{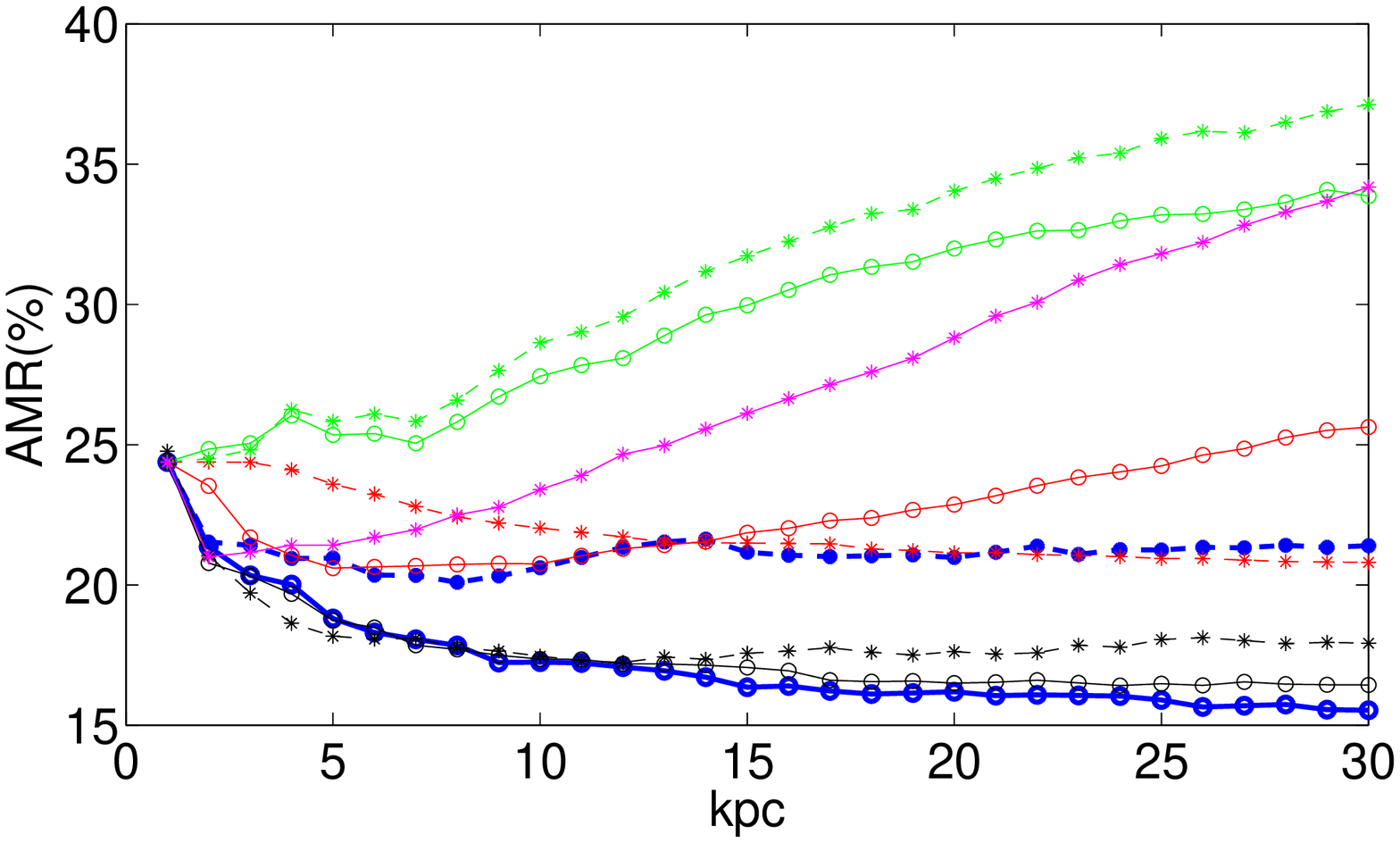}}
\subfloat[T3 dataset (p=10)]{\includegraphics[width=0.33\textwidth]{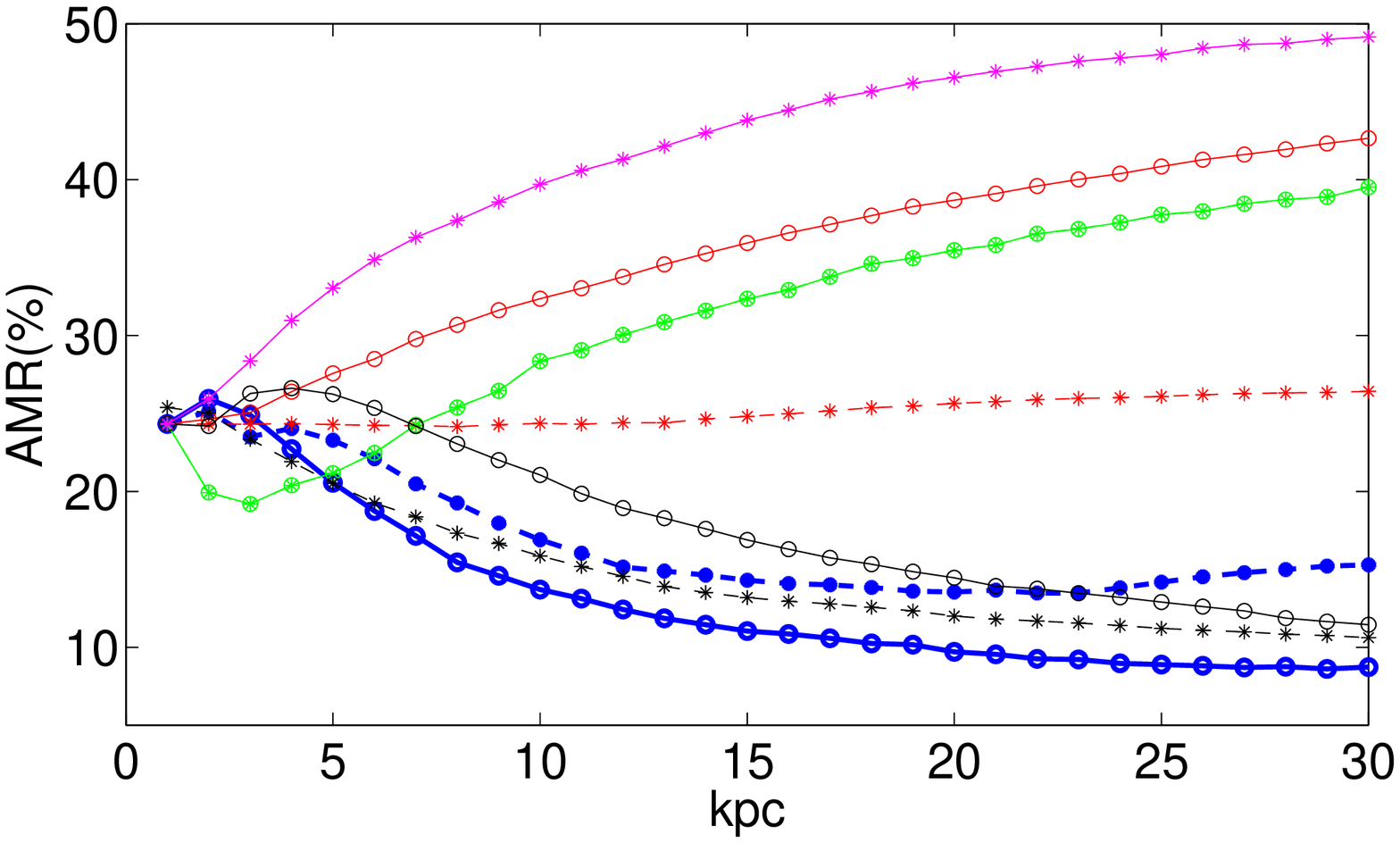}}

\subfloat[T4 dataset (p=2)]{\includegraphics[width=0.33\textwidth]{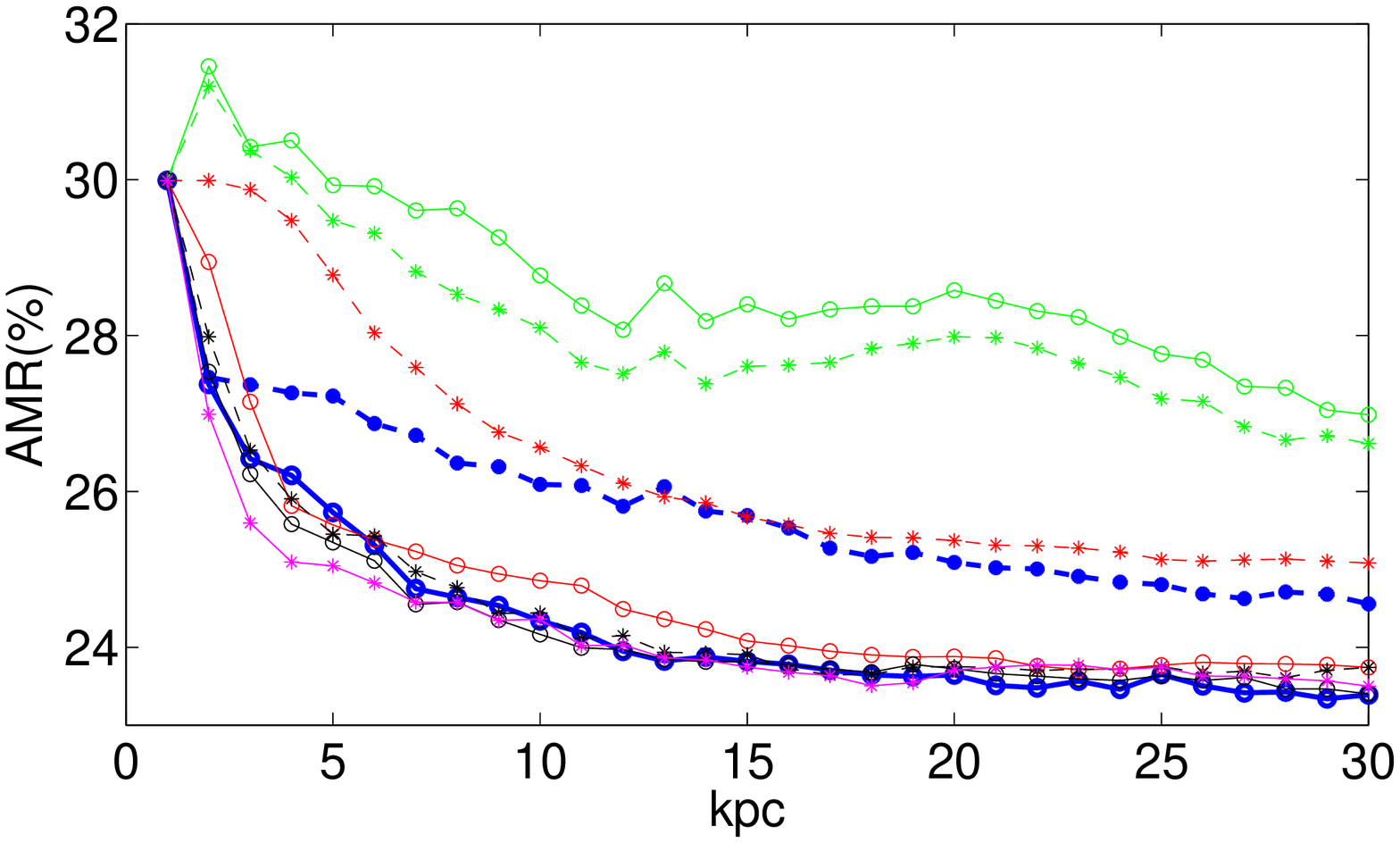}}
\subfloat[T4 dataset (p=5)]{\includegraphics[width=0.33\textwidth]{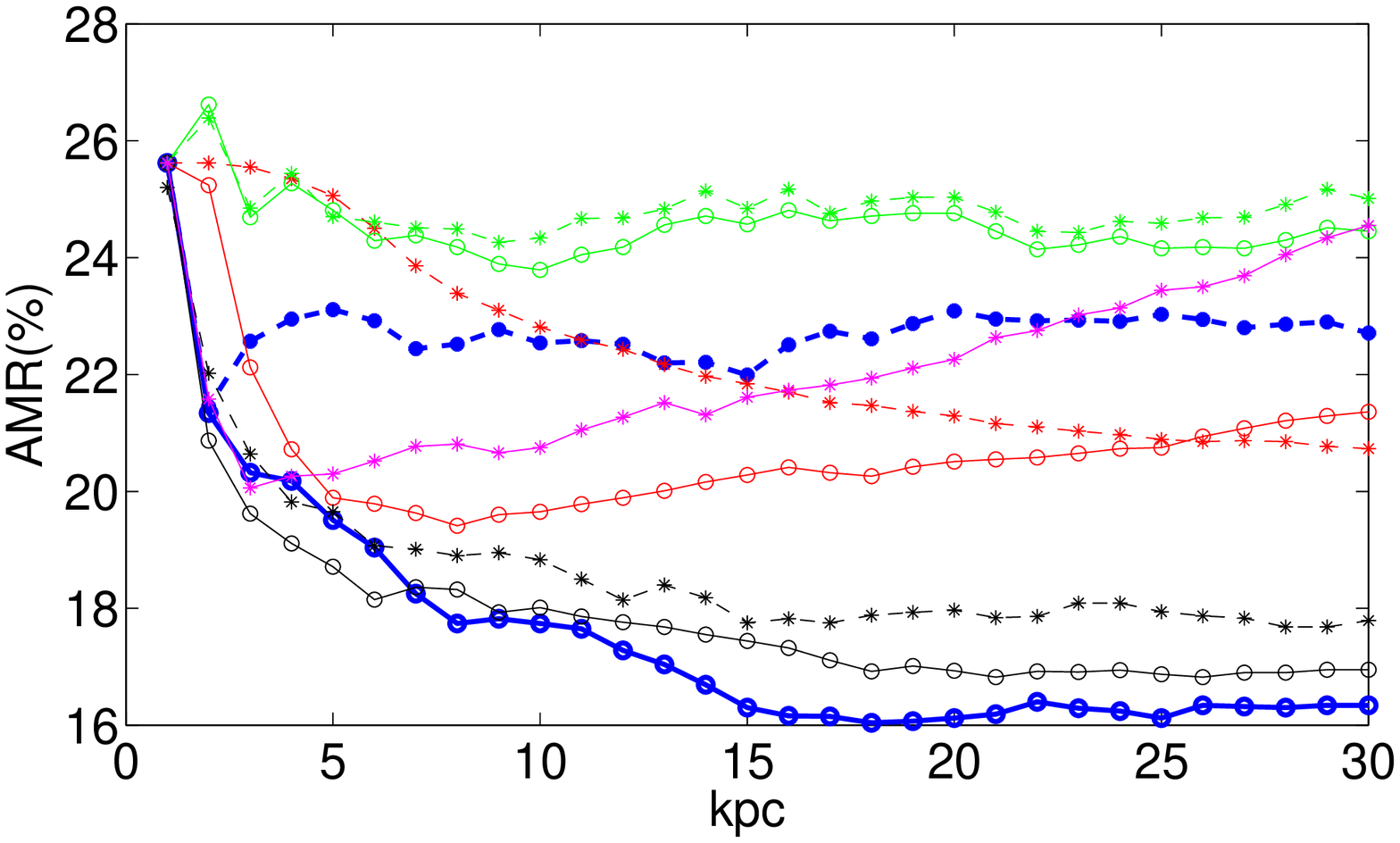}}
\subfloat[T4 dataset (p=10)]{\includegraphics[width=0.33\textwidth]{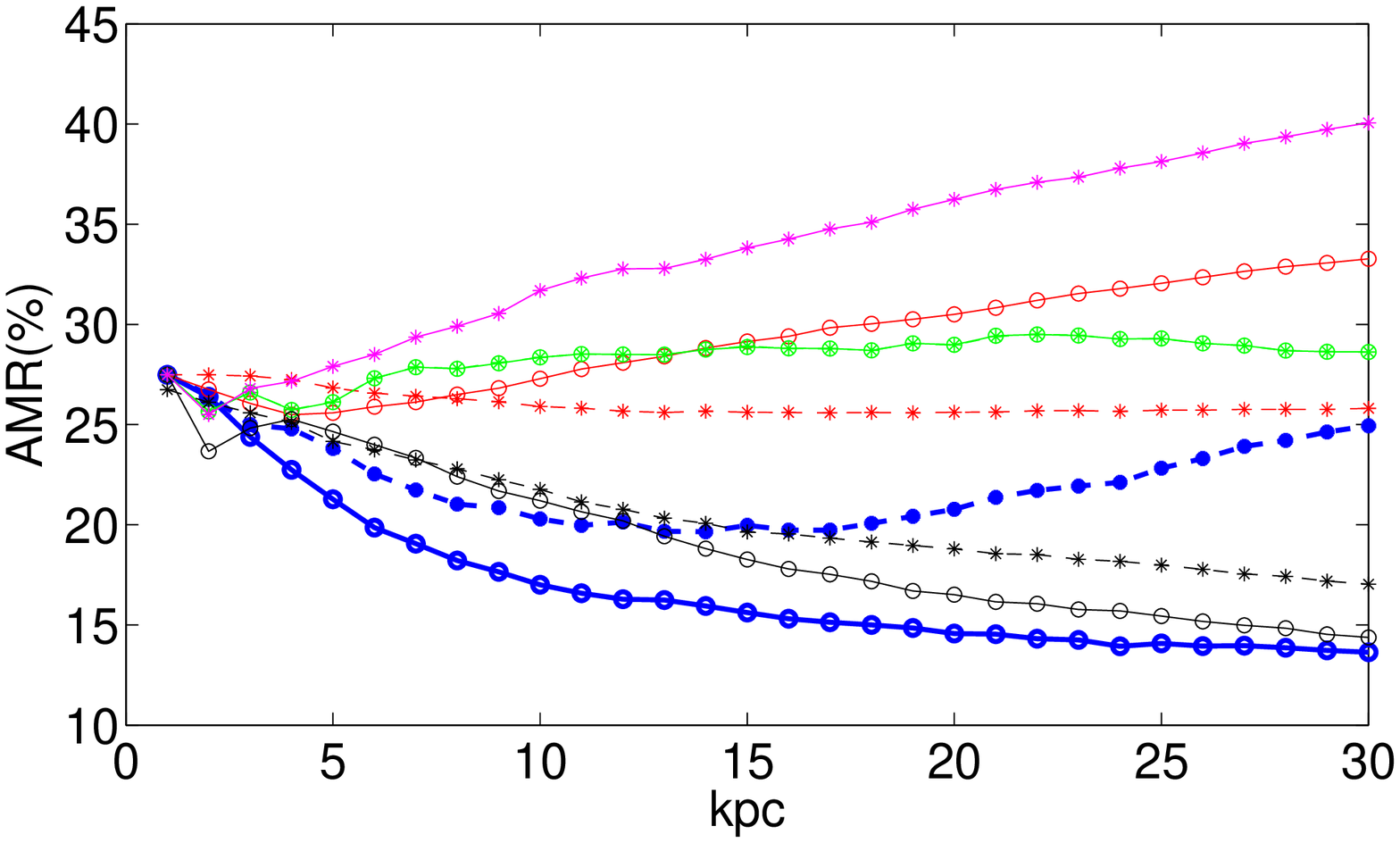}}


\caption{The performance curves of the kNN-based classifiers with respect to \textit{kpc} on the 4 types of synthetic datasets and several real datasets.}
\label{fig_kpc2}
\end{figure*}

\begin{figure*}[!t]
\centering
\subfloat{\includegraphics[width=0.9\textwidth]{legendk}}

\addtocounter{subfigure}{-1}
\subfloat[Bupaliver]{\includegraphics[width=0.33\textwidth]{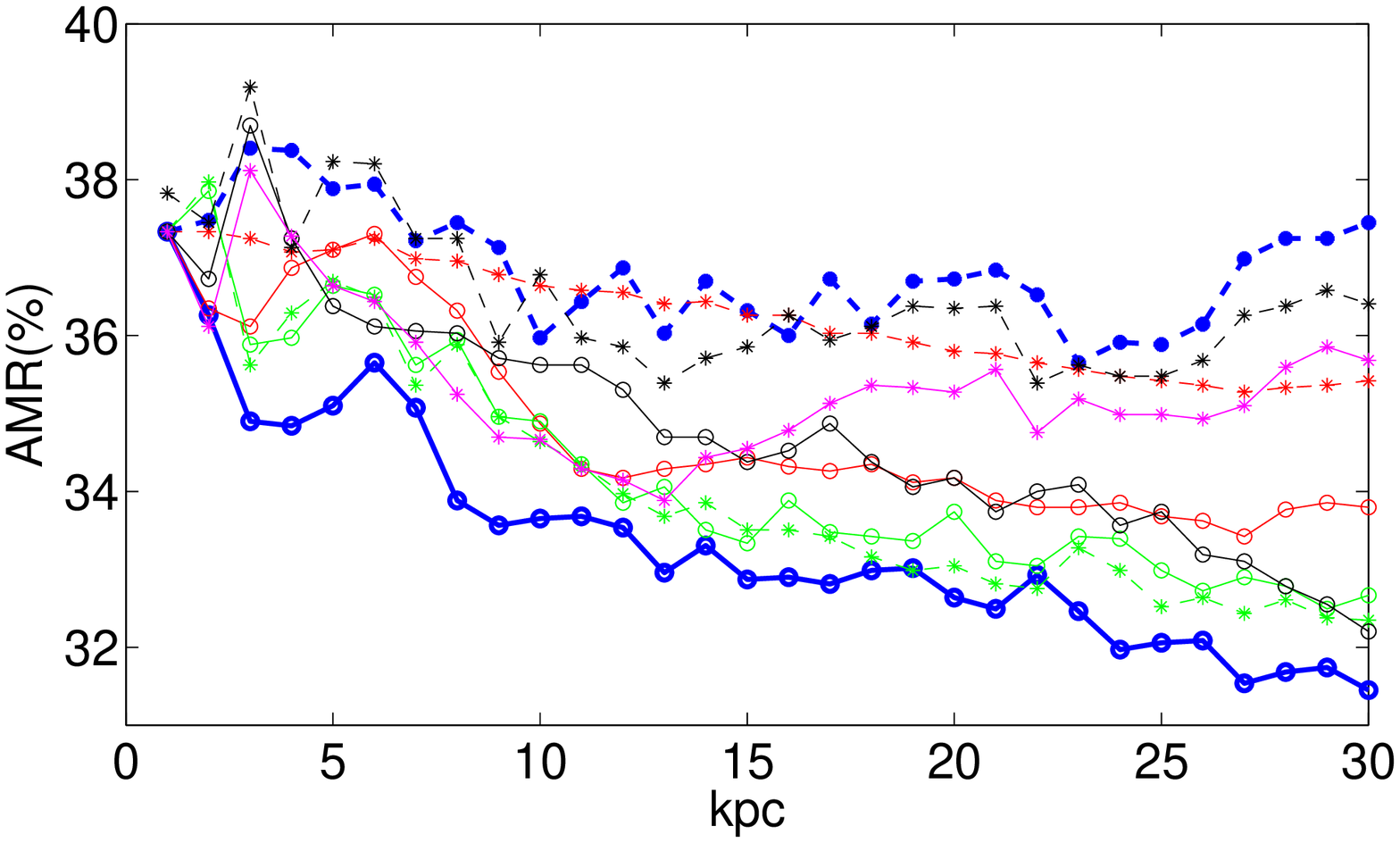}}
\subfloat[Climate]{\includegraphics[width=0.33\textwidth]{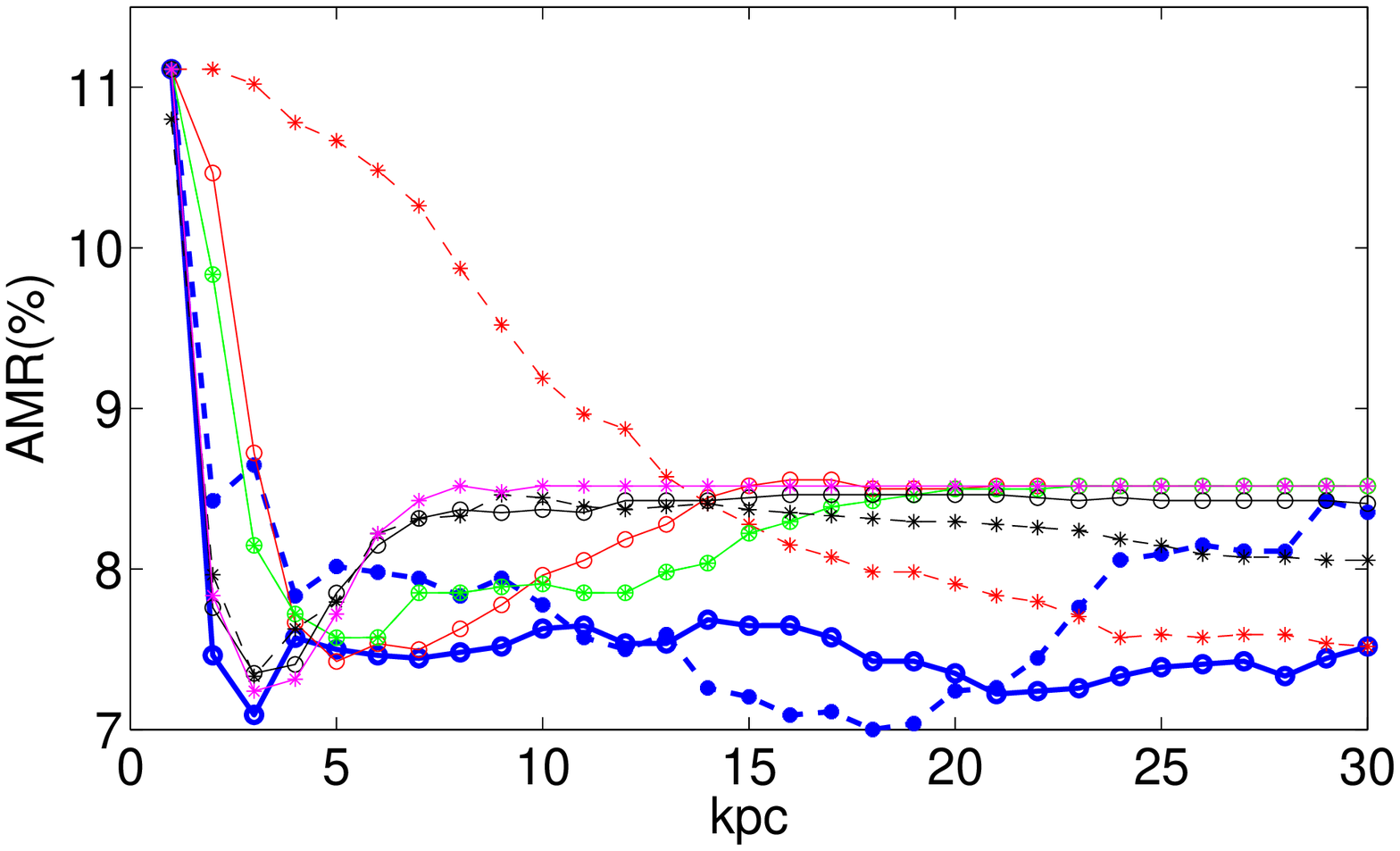}}
\subfloat[Sonar]{\includegraphics[width=0.33\textwidth]{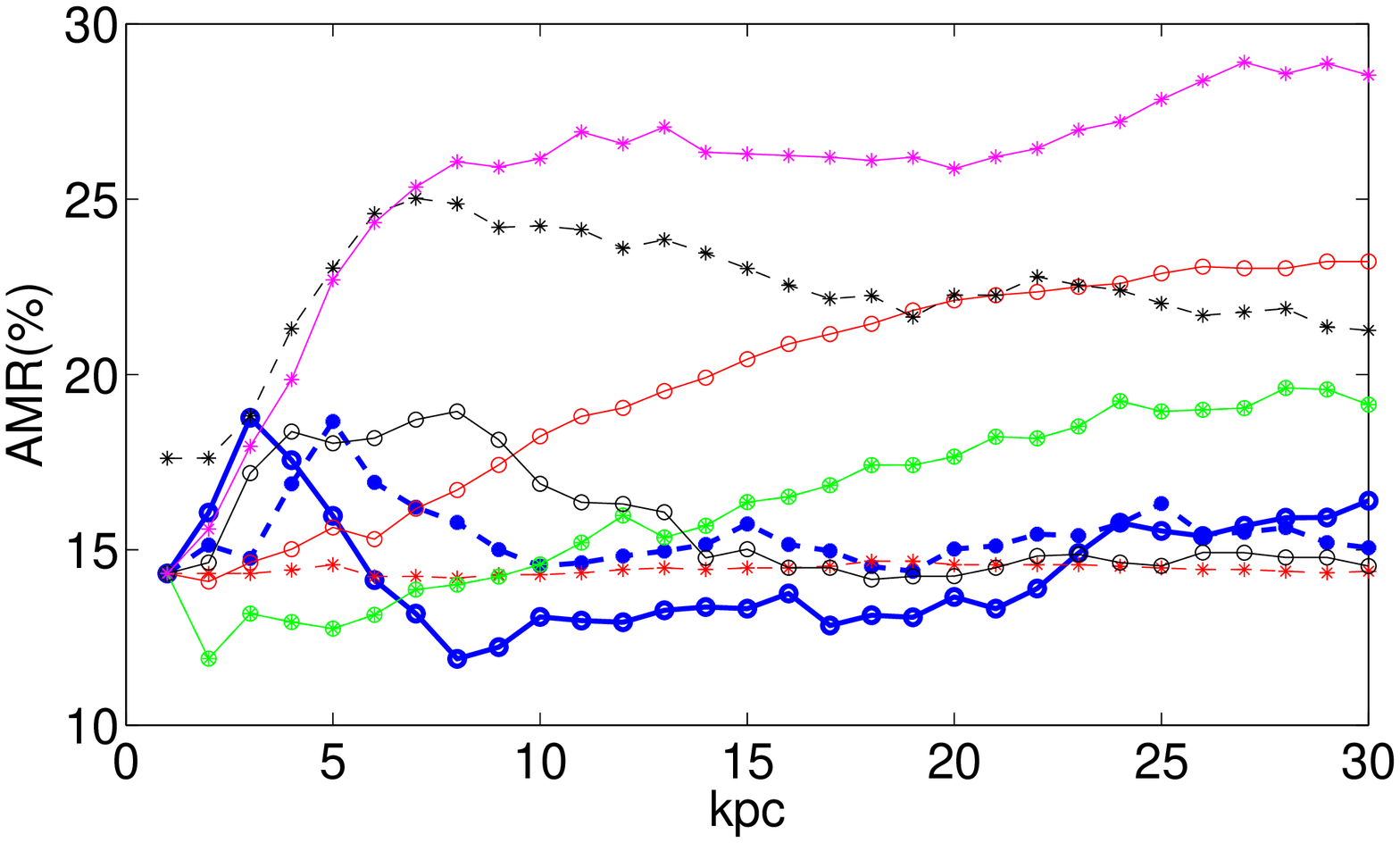}}

\subfloat[Vertebral2]{\includegraphics[width=0.33\textwidth]{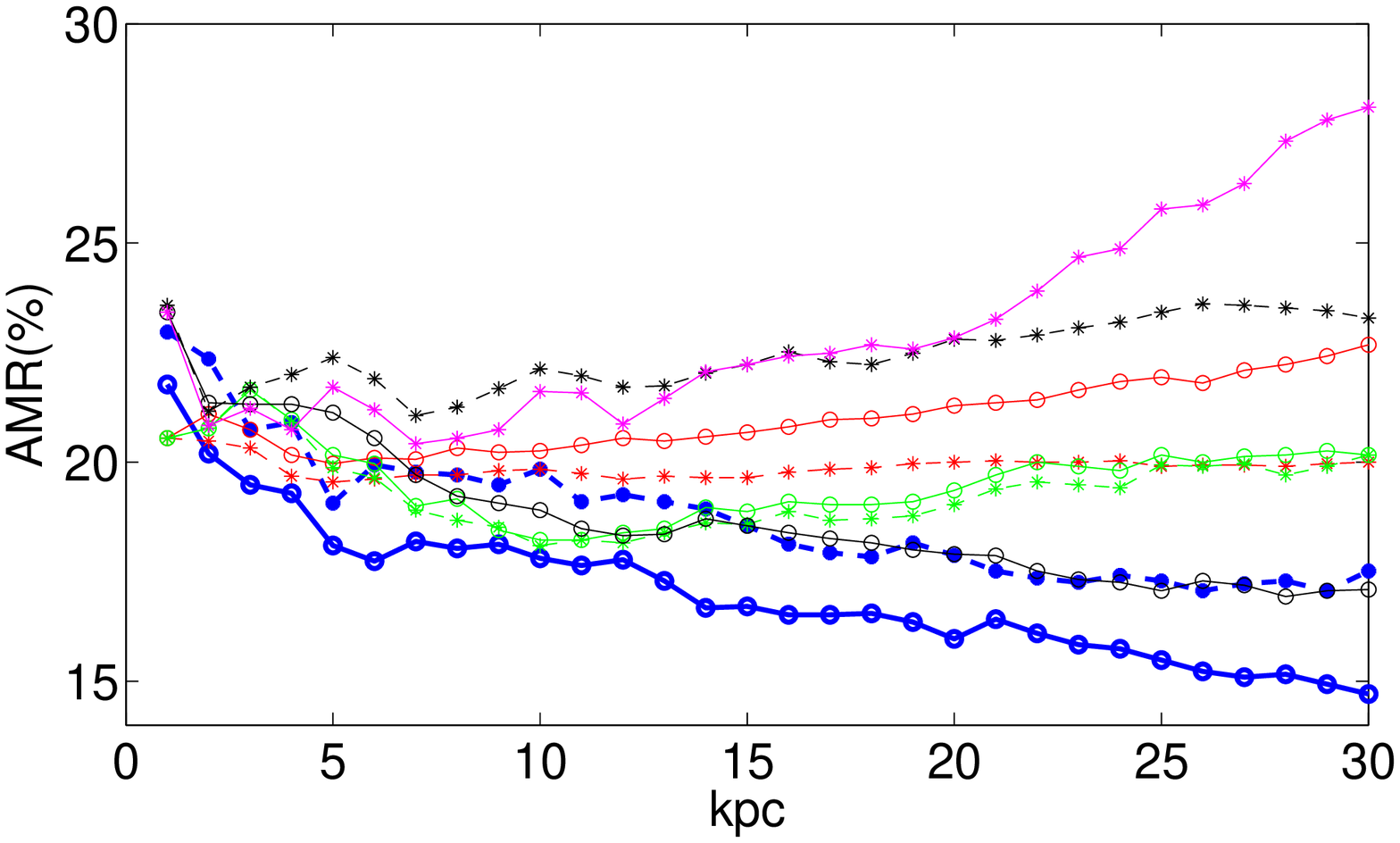}}
\subfloat[Vehicle]{\includegraphics[width=0.33\textwidth]{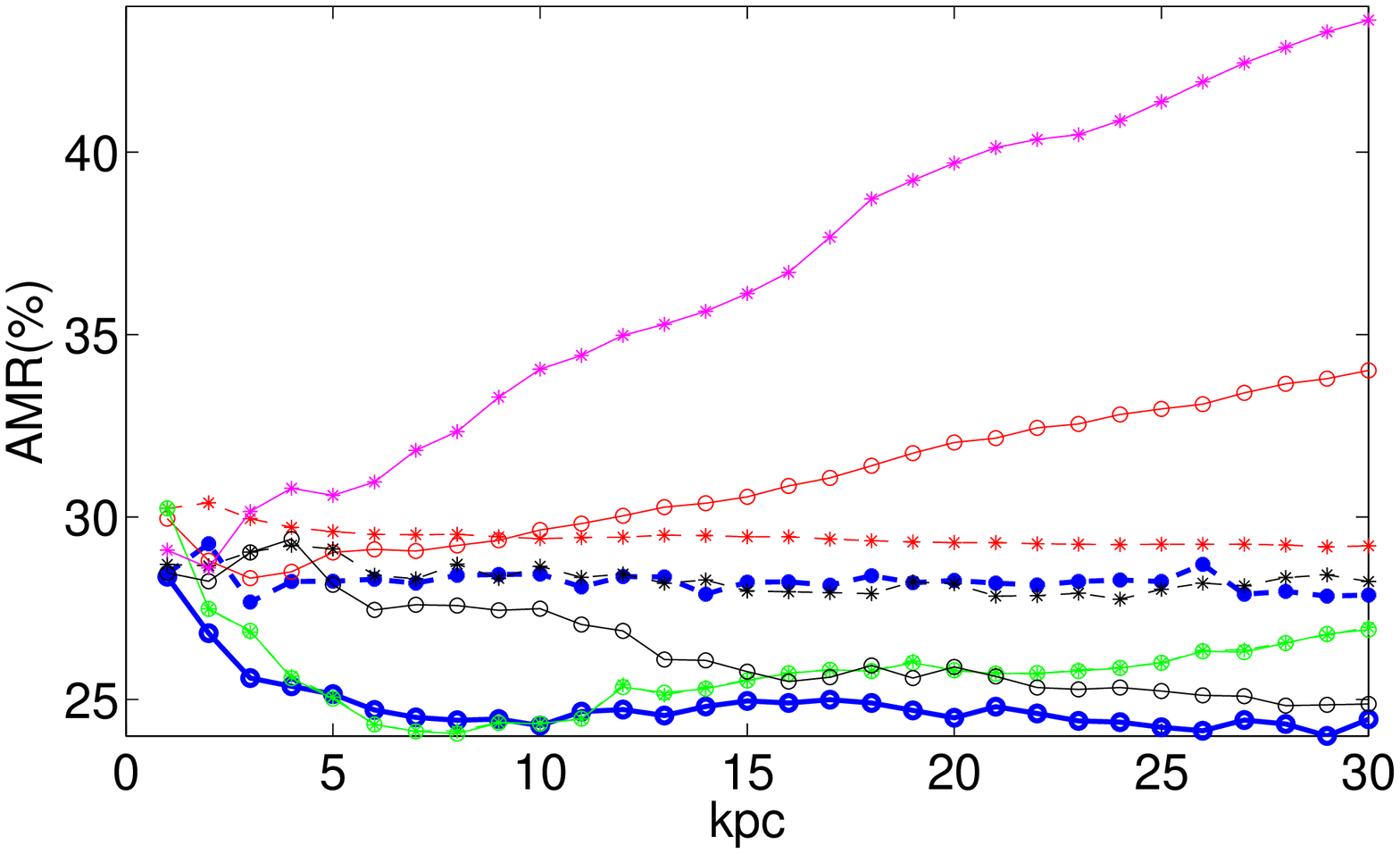}}
\subfloat[Wine]{\includegraphics[width=0.33\textwidth]{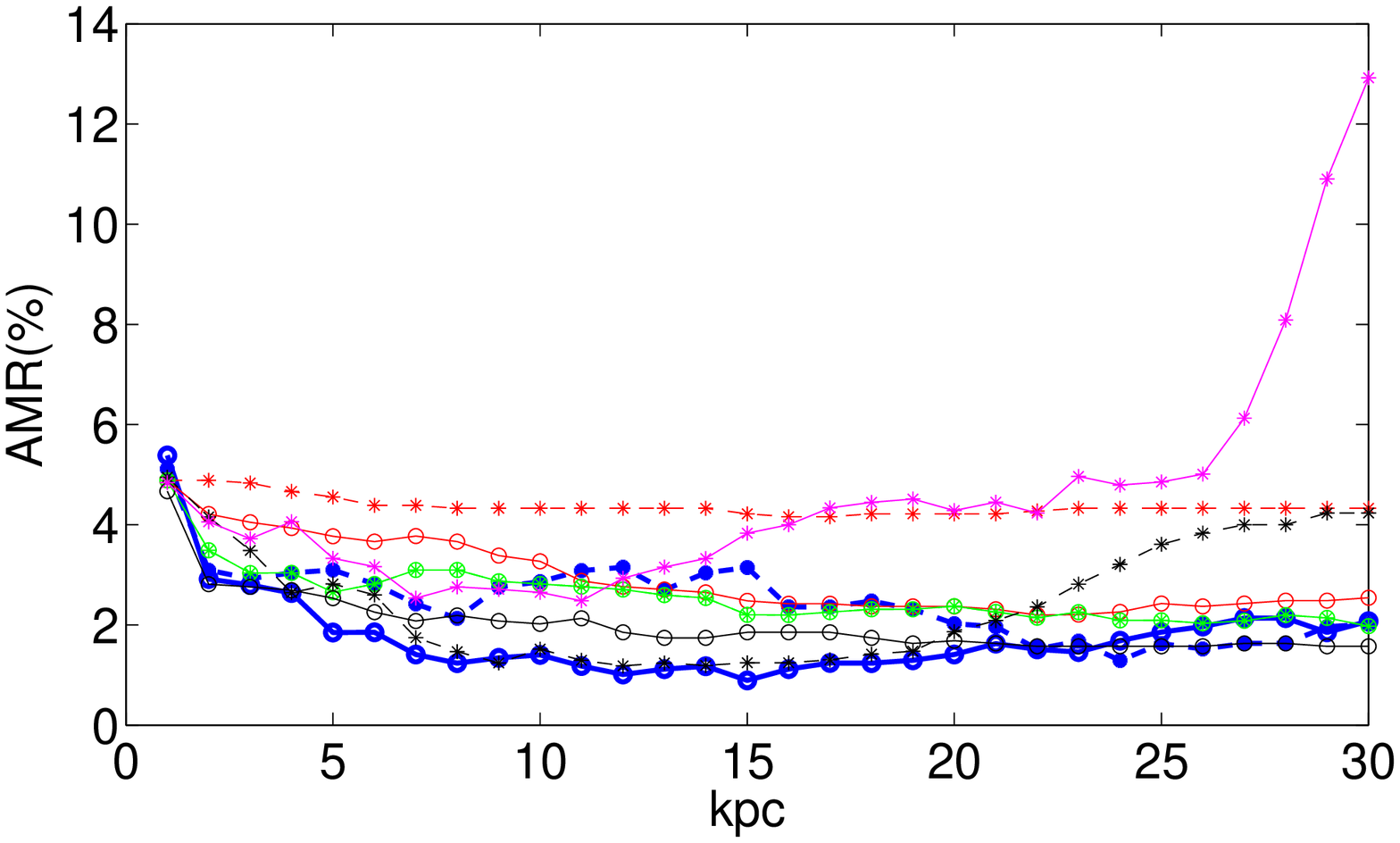}}

\caption{The performance curves of the kNN-based classifiers with respect to \textit{kpc} on the 4 types of synthetic datasets and several real datasets.}
\label{fig_kpc_real}
\end{figure*}

For the T1 dataset, the complex boundary may make the neighborhood of samples around the boundary be dominated by samples from the other class; this will frustrate the kNN rules, especially the neighborhood is large. A small neighborhood can alleviate this dominance, but easily suffers from randomness of samples. Thus, the optimal choice of \textit{kpc} should not be too small or too large, as depicted in Fig. \ref{T2_2}. However, in high-dimensional cases, due to the sparsity of samples, very few points will be sampled around the boundary, the dominance from a different class should be weaker, and a large neighborhood may reduce the impact of the randomness of samples; thus a larger neighborhood is favorable, as shown in \ref{T2_5} and \ref{T2_10}. It can also be seen that the LD-kNN(GME) with an appropriate neighborhood size is more effective than other classifiers. The LD-kNN(KDE) is also as effective as LD-kNN(GME) in low dimensional cases (p=2); while in high dimensional cases (p=5 and p=10), it can not be so effective as LD-kNN(GME). It can be explained that, in low dimensional cases, the samples are enough for the nonparametric estimation of local distribution; while in high dimensional cases, due to the sparsity of samples the nonparametric estimation as KDE may fail to achieve an effective estimation as a parametric estimation as GME.

For the T2, T3 and T4 datsets, we can see that LD-kNN(GME) usually favors a larger neighborhood size and that LD-kNN(GME) is more effective than other kNN rules in most cases. In low-dimensional cases, the samples are dense, other kNN rules can also effectively simulate the local distribution; thus in Fig. \ref{fig_kpc2} the advantages of LD-kNN(GME) over other kNN rules are not so significant when p=2, while more significant when p=5 or 10. In T2 dataset, the Bayesian decision boundary is linear and the distributions of the two classes are symmetrical about the boundary; most of the kNN rules can be effective for this simple classification problem, LD-kNN(GME) can not show any advantages over some other kNN rules. The LD-kNN(KDE) does not perform so effectively as LD-kNN(GME) on these types of datasets. We believe that assuming a simple model such as Gaussian model in a small local region can be more effective than a more complicated probabilistic model as estimating the local distribution through KDE.

\begin{figure*}[!t]
\centering
\subfloat{\includegraphics[width=0.9\textwidth]{legendp}}
\addtocounter{subfigure}{-1}

\subfloat[T1 dataset]{\includegraphics[width=0.24\textwidth]{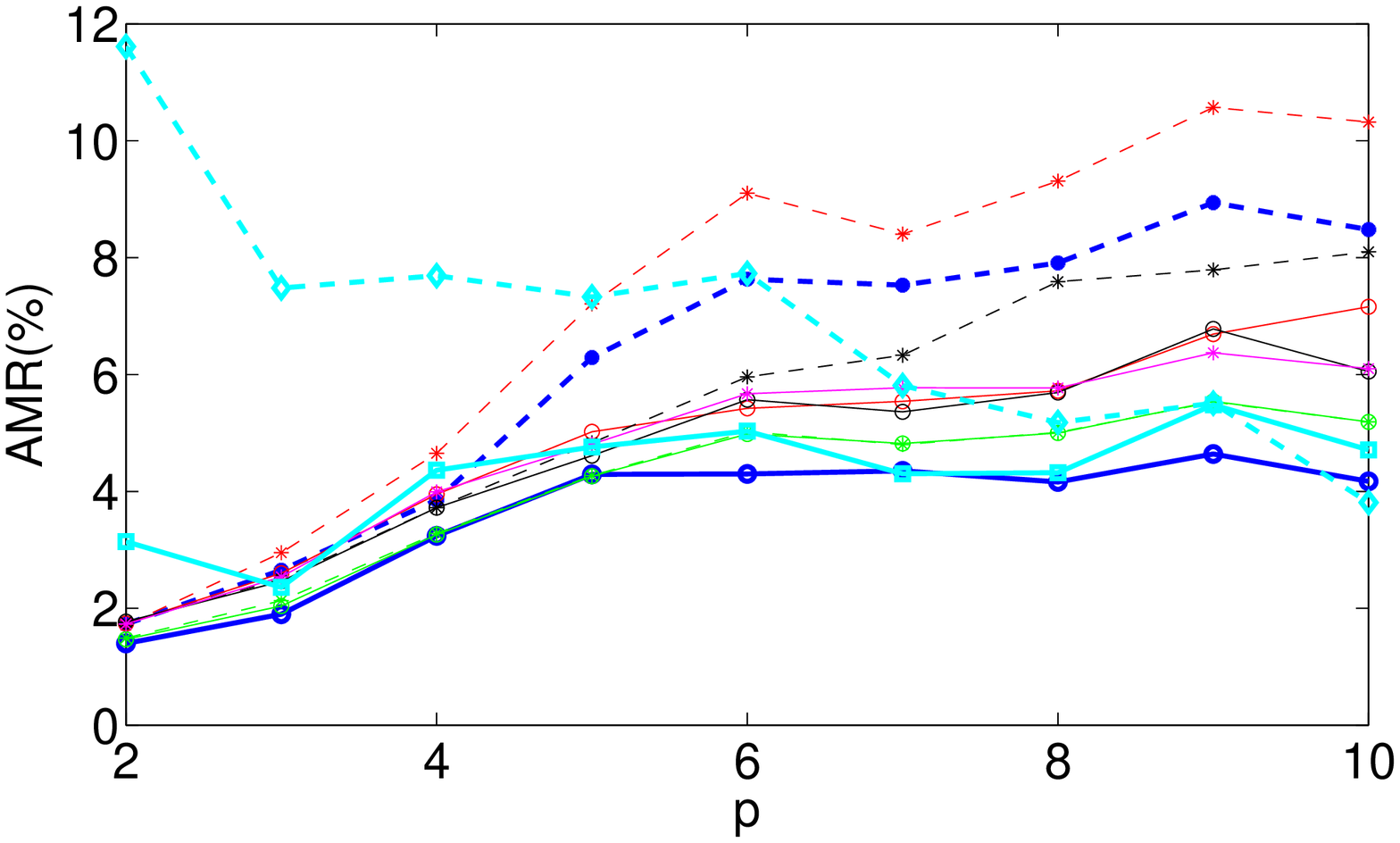}
\label{fig_p_T2}}
\subfloat[T2 dataset]{\includegraphics[width=0.24\textwidth]{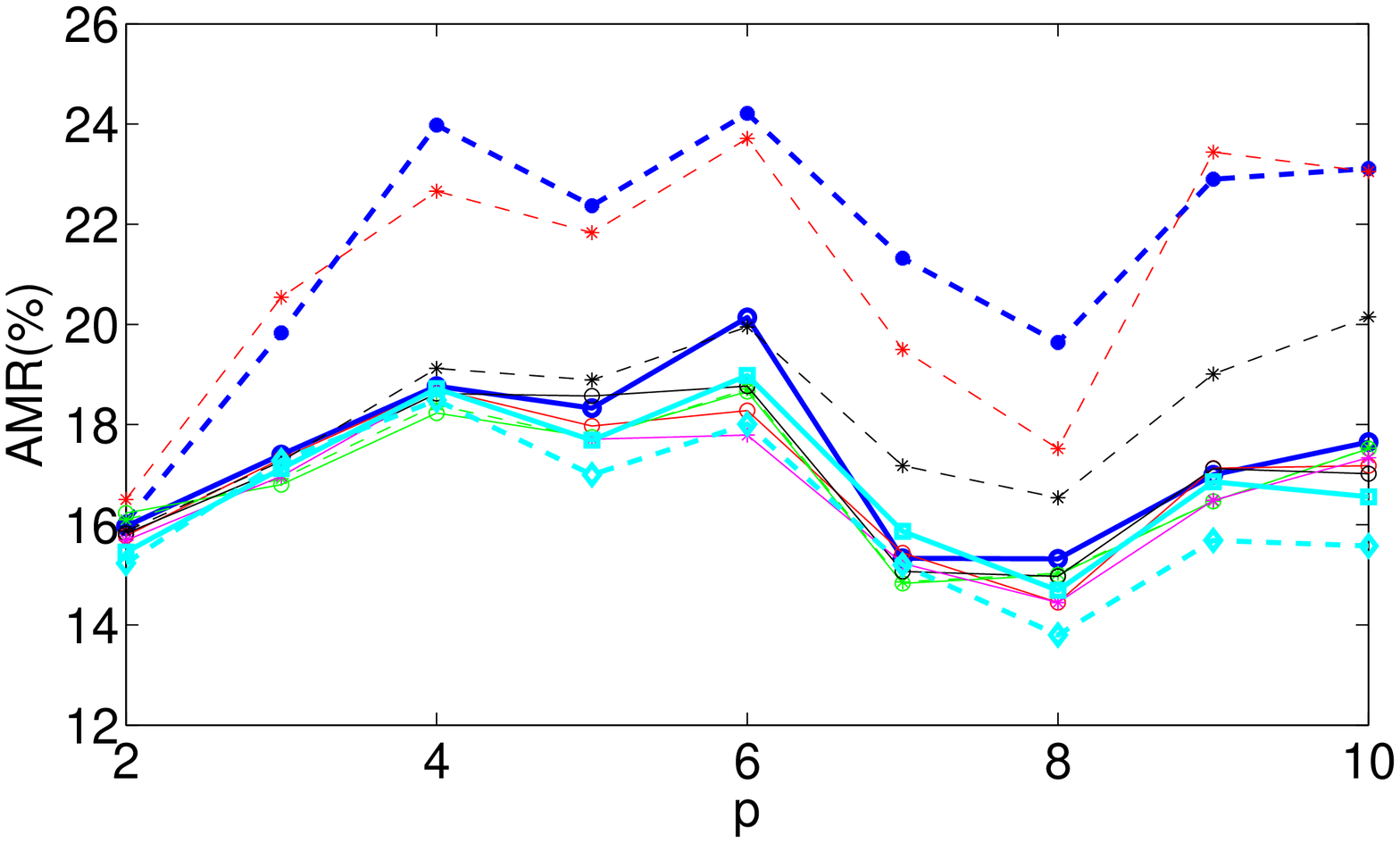}
\label{fig_p_T3}}
\subfloat[T3 dataset]{\includegraphics[width=0.24\textwidth]{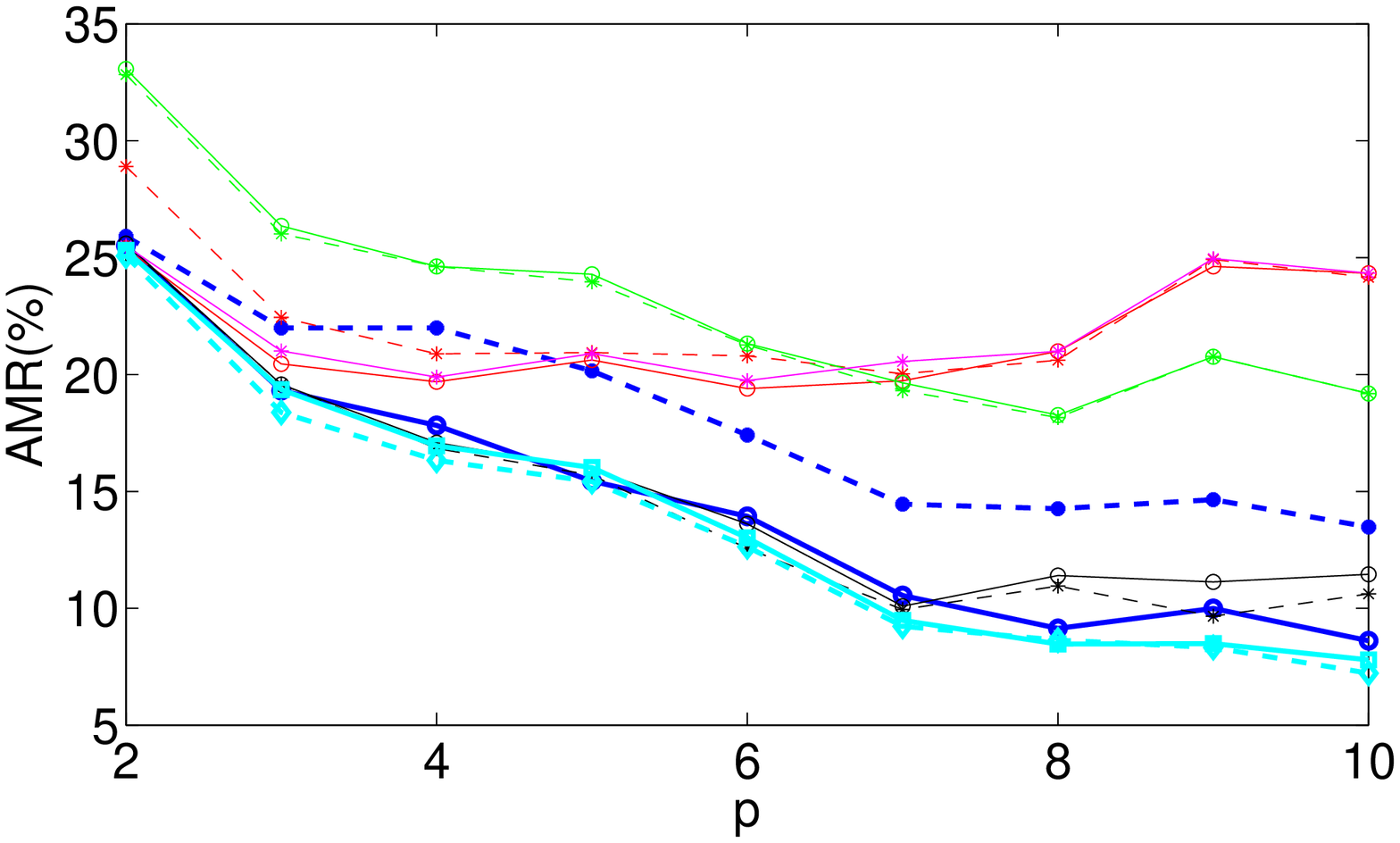}
\label{fig_p_T4}}
\subfloat[T4 dataset]{\includegraphics[width=0.24\textwidth]{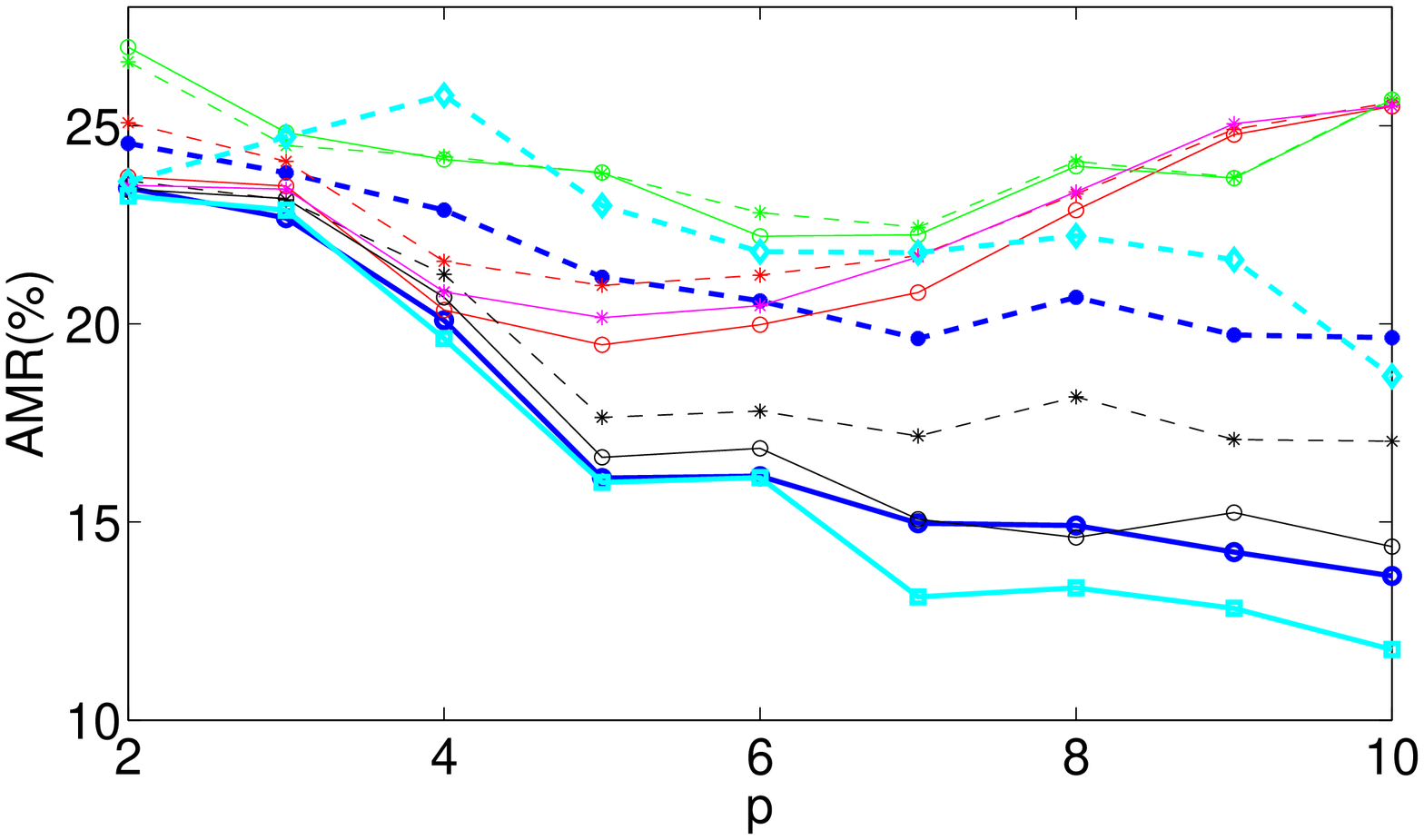}
\label{fig_p_T6}}


\caption{The classification performance varied with respect to the dimension for the classifiers on the four types of datasets.}
\label{fig_p}
\end{figure*}

For a real-world classification problem, the true distribution and the relationship between attributes is usually unknown, the influences of parameter \textit{kpc} on classification results are usually irregular, and the neighborhood size should be selected according to the testing results. Fig. \ref{fig_kpc_real} shows the performance curves with respect to \textit{kpc} of the  kNN-based methods on several real datasets. Because different real datasets usually have different distributions and different sizes, the curves for LD-kNN(GME) are usually different as depicted in the last row of Fig \ref{fig_kpc_real}. However, the performance curves on these datasets show that, on average the LD-kNN(GME) method can be quite effective for these real problems and that a somewhat larger neighborhood should be selected for LD-kNN(GME).

%
%
%
%
%
%

\subsection{Experiment II }
The purpose of experiment II is to investigate the scalability of dimension (or dimensional robustness) of the methods; i.e., how well a method performs while the dimension of datasets is increasing. In this experiment we implement the classification algorithms on each type of synthetic dataset with 500 samples for each class and with the dimension \textit{p} ranging from 2 to 10. The AMR of 10 trials of stratified 5-fold cross validation is obtained for each classifier. The \textit{kpc} for a kNN-based classifier is selected according to its performance in Experiment I. In addition, we have also employed SVM and naive Bayesian classifiers (NBC) as controls due to their acknowledged dimensional scalability.

Fig. \ref{fig_p} plots the classification performance as a function of the dimension \textit{p} for each of the classifiers on the four types of synthetic datasets. From Fig. \ref{fig_p}, we can see that the LD-kNN(GME) is generally more dimensionally robust than other kNN-based classifiers.

Fig. \ref{fig_p_T2} shows the uptrends of AMR for most of the classifiers with the dimension increasing on the corresponding datasets T1; it can be seen that the uptrend for LD-kNN(GME) is much slower than that for other kNN-based classifiers and SVM, which demonstrates the dimensional robustness of LD-kNN(GME). The NBC shows a downtrend of AMR with the dimension increasing; however it can not perform so effectively as LD-kNN(GME) due to the violation of the conditional independence assumption between attributes.
For the T2 dataset, Fig. \ref{fig_p_T3} demonstrates the similarity trends in terms of AMR among these classifiers with the dimension increasing. The LD-kNN(KDE) and DW2-kNN can be seen to perform not so well as the other classifiers.
From Fig. \ref{fig_p_T4}, LD-kNN(GME), SVM-kNN(RBF), SVM-kNN(Poly), SVM and NBC show similar effectiveness and robustness to the dimension on T3 datasets, and it can be seen that they perform more effective than other classifiers. In addition, the V-kNN, DW1-kNN and DW2-kNN are not so dimensional robustness.
For the T4 dataset, Fig. \ref{fig_p_T6} shows that the performance of LD-kNN(GME) is also as robust to the dimension as SVM.

From the above analysis it can be seen that the dimensional robustness of LD-kNN(GME) can be generally comparable to that of NBC and SVM on these datasets. On T1 and T4 datasets where the class conditional independence assumption is violated for NBC, the LD-kNN(GME) can be more effective than NBC and demonstrate comparable robustness with the SVM. Additionally, on all these datasets, the performances of LD-kNN(GME) are even more robust than or at least equal to SVM-kNN which used to be the state-of-the-art methods for classification.

\subsection{Experiment III }
Experiment III is designed to evaluated the efficiency of the LD-kNN methods when applied in datasets with different sizes. Because the time consumption for classification is related to the dataset size and the classifier more than the data distribution, synthetic datasets with different sizes can be applicable for efficiency tests. In this experiment, three factors, the neighborhood size for kNN-based classifiers, the size of the training set and the dimensionality of the dataset are investigated for the efficiency tests. Three experimental tasks corresponding to the three factors are implemented on a PC with one Pentium(R) 2.50GHz processor and one 2.00GB RAM. The computing software is Matlab R2010a in Windows XP OS.

\begin{figure*}[!t]
\centering
\subfloat[]{\includegraphics[width=0.33\textwidth]{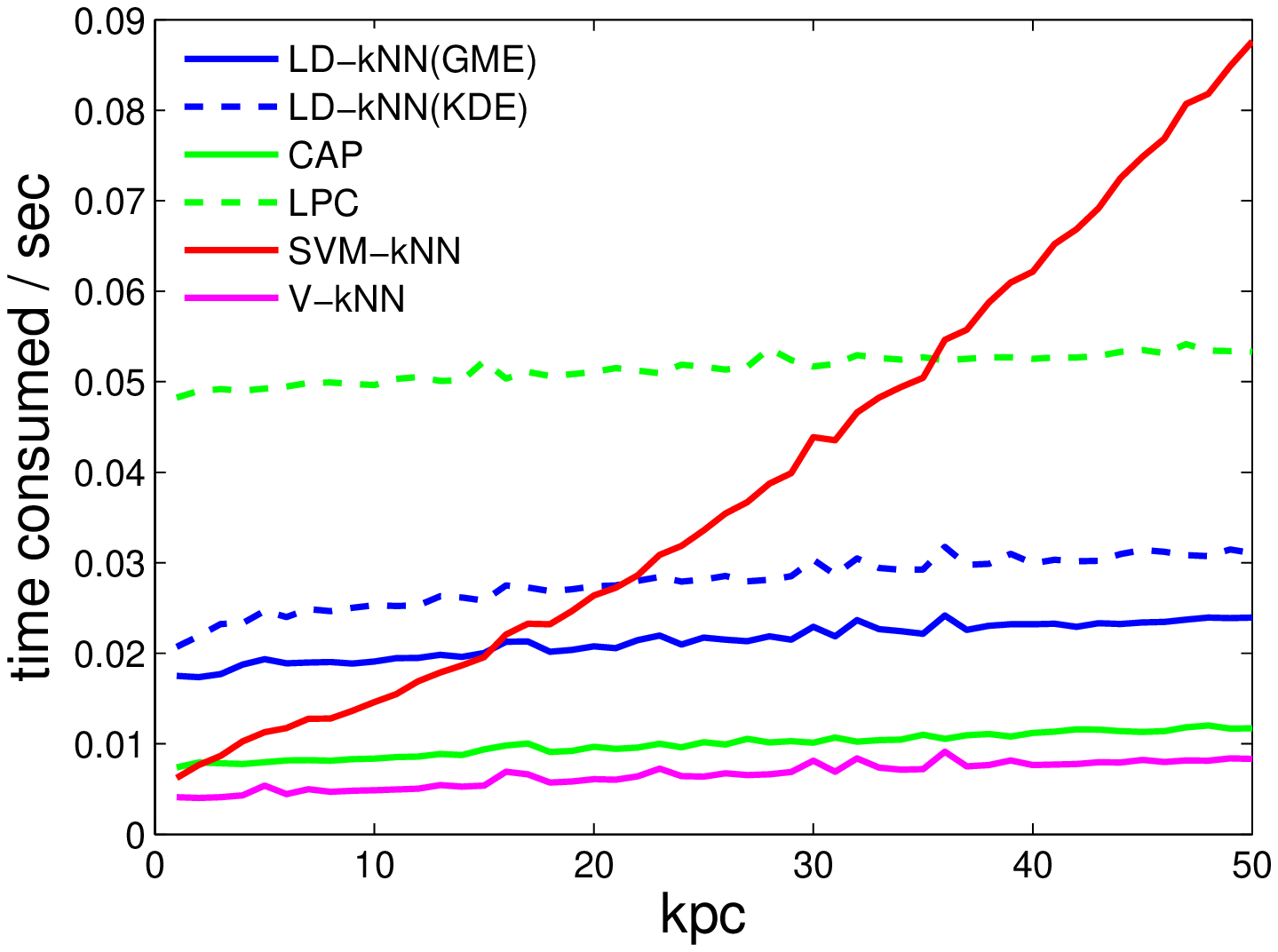}
\label{timek}}
\subfloat[]{\includegraphics[ width=0.33\textwidth]{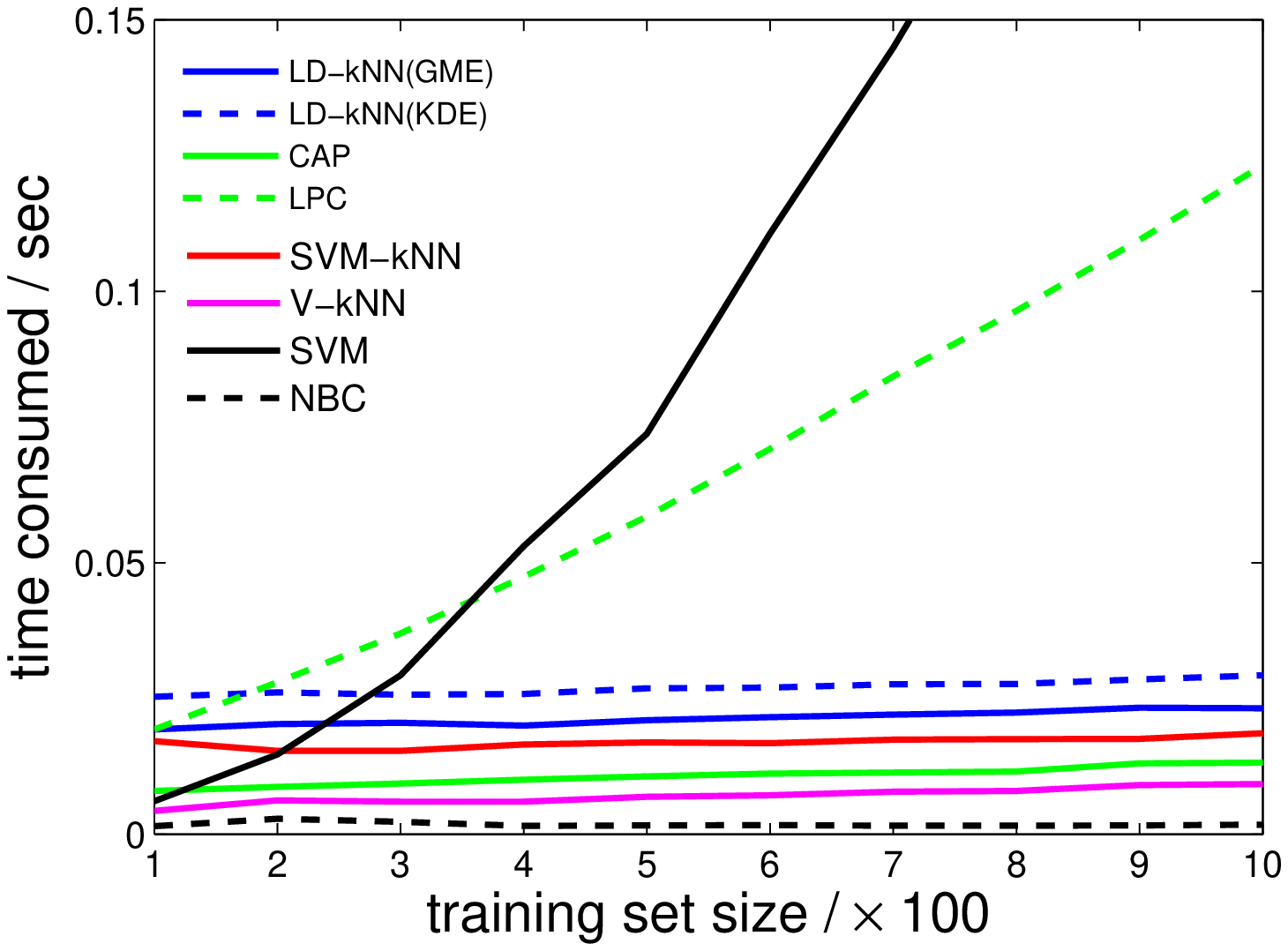}
\label{timen}}
\subfloat[]{\includegraphics[width=0.33\textwidth]{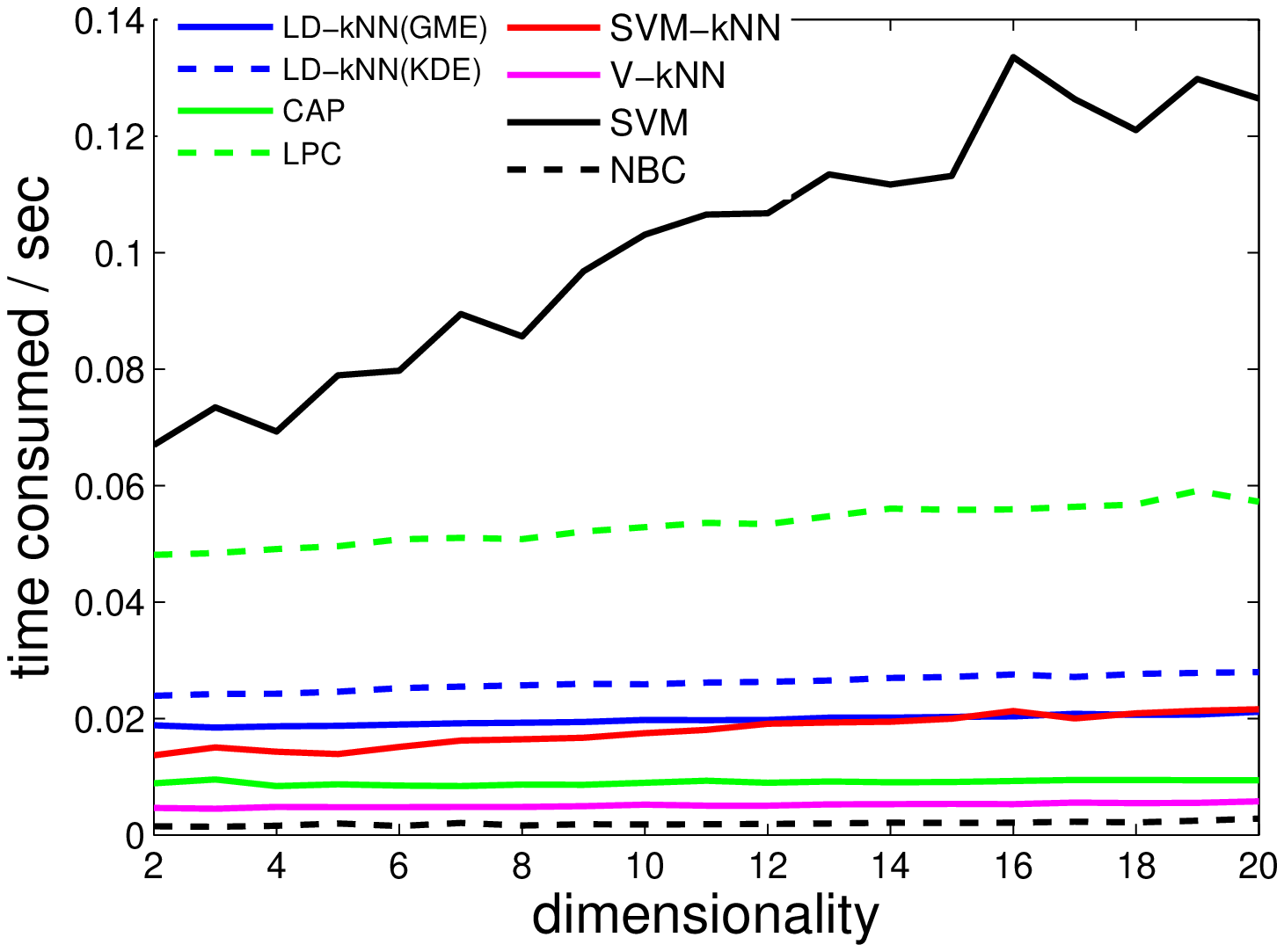}
\label{timed}}

\caption{The classification efficiency varied with respect to (a) neighborhood size, (b) training set size and (c) dataset dimensionality. }
\label{efficiency}
\end{figure*}

To investigate the influence of neighborhood size on the efficiency for kNN-based classifiers, we use a synthetic training set with 500 samples and a test set with 100 samples; both are 5-dimensional T2 datasets. The time consumptions for the kNN-based classifiers vary with the neighborhood size (i.e. $kpc$) as depicted in Fig. \ref{timek}, where we can see that with the neighborhood size increasing, the time consumptions of all these classifiers except SVM-kNN increase very slowly; while that of SVM-kNN increases more sharply, in agreement with the complexity analysis in Section \ref{complexity} that the SVM-kNN has a computational complexity $O(k^2)$ while other kNN-based classifiers is $O(k)$ if $m$ and $d$ are constant. From Fig. \ref{timek}, it can also be seen that with a constant neighborhood size, LD-kNN(GME) is less time-consuming than LD-kNN(KDE)and LPC, and more time-consuming than V-kNN, DW-kNN and CAP.

To investigate how efficiently the classifiers perform with the training set size increasing, we employ a training set with sample size varying from 100 to 1000 and the test set with 100 samples; both are also 5-dimensional T2 datasets; the neighborhood size is set $kpc=10$ for kNN-based classifiers. We have plot the variation curves (time consumption vs. training set size) in Fig. \ref{timen}. From Fig. \ref{timen}, with $kpc$ increasing, the time consumptions of all these classifiers except SVM and LPC increase so slightly that they remain visually stable. The time consumption of SVM increases most sharply and then LPC among these classifiers. This is because the computational complexity of SVM is between $O(m^2d)$ and $O(m^3+m^2d)$ as shown in \cite{bottou2007support} and LPC has an additional prior probability estimation process whose computational complexity is $O(k_0m^2)$ as described in \cite{li2008nearest}; while all other kNN-based classifiers just have a linear complexity with respect to training set size $m$. In addition, NBC and V-kNN are both quite efficient due to their simplicity. LD-kNN(GME) is a little more efficient than LD-kNN(KDE) but less efficient than NBC and other kNN-based classifiers.

As for the dataset dimensionality, we also use a training set and a test set respectively with 500 samples and 100 samples as before; $kpc=10$ is set for kNN-based classifiers; the only variable is the dimensionality that we vary from 2 to 20. The efficiency-dimensionality curves for different classifiers are shown in Fig. \ref{timed}. From Fig. \ref{timed}, we can see that the time consumptions for all classifiers show roughly linear increases with the dimensionality increasing; the increase of SVM is much sharper than that of other classifiers. In addition, LD-kNN(GME) and LD-kNN(KDE) are not so efficient as NBC, V-kNN and CAP; however, the increase rates among the five classifiers do not shown any significance difference.

All the results in Fig. \ref{efficiency} corroborate the computational complexity analysis in Section \ref{complexity}. From the results and analysis, the time consumption of LD-kNN does not increase much with the increase of neighborhood size, training set size and dataset dimensionality as is the case of NBC, V-kNN and CAP.

\subsection{Experiment IV }

The purpose of Experiment IV is to investigate the classification performance of the proposed method on the real datasets by comparing it with competing methods. In this experiment, we implement LD-kNN rules on the 27 UCI datasets for classification. The \textit{kpc} value for a kNN-based rule and a dataset is selected according to its performance in Experiment I.

Following experimental testing, the classification results in terms of AMR and AF1 on real datasets for all the classifiers are shown in Table \ref{tab:results} where the upper half is AMR and the lower half is AF1. From Table \ref{tab:results} we can see that LD-kNN(GME) performs best for 8 datasets in terms of both AMR and AF1, which is more than other classifiers. The overall average AMR and the corresponding average rank of LD-kNN(GME) on these datasets are 15.27\% and 3.41 respectively, lower than all other classifiers; similar results can be achieved when using AF1 as the indicator. We have compared LD-kNN(GME) with each of other classifiers in terms of AMR and AF1, the number of wins, ties and losses are displayed in the last row of the corresponding sub-table in Table \ref{tab:results}; the LD-kNN(GME) has a significant advantage in each pairwise comparisons. These results imply that the proposed LD-kNN(GME) may be a promising method for classification.

\begin{table*}
  \centering
  \caption{The AMR (\%) and AF1 for all the classifiers on the 27 UCI datasets and the corresponding statistical results. The best performance for each dataset is described in bold-face.}
    \begin{tabular}{l||p{25pt}p{25pt}p{25pt}p{25pt}p{25pt}p{25pt}p{25pt}p{25pt}p{25pt}||p{25pt}p{25pt}p{25pt}}
    \toprule
    Datasets & LD-(GME) & LD-(KDE) & V-kNN & DW1-kNN & DW2-kNN & CAP   & LPC   & SVM-(RBF) & SVM-(Poly) & RBF-SVM & GME-NBC & KDE-NBC \\
    \midrule
           & \multicolumn{12}{c}{AMR} \\
    \midrule
    Blood & 20.67 & 22.01 & 20.64 & 20.43 & 24.33 & 20.56 & 20.37 & \textbf{20.25} & 20.47 & 21.58 & 23.44 & 22.45 \\
    Bupaliver & 31.62 & 33.07 & 34.17 & 33.51 & 36.41 & 32.52 & 32.70 & 32.72 & 35.13 & \textbf{30.17} & 44.35 & 46.38 \\
    Cardio1 & 19.83 & 22.74 & 25.15 & 22.52 & 22.20 & 20.86 & 20.71 & 20.14 & 24.14 & \textbf{18.42} & 27.01 & 31.62 \\
    Cardio2 & \textbf{8.23} & 8.52  & 9.24  & 8.70  & 8.76  & 8.25  & 8.25  & 9.11  & 9.04  & 8.83  & 17.65 & 21.93 \\
    Climate & 7.30  & 6.93  & 7.41  & 7.61  & 7.63  & 7.63  & 7.63  & 7.54  & 7.50  & \textbf{5.35} & 5.39  & 7.44 \\
    Dermatology & \textbf{2.10} & 3.91  & 4.21  & 4.18  & 4.18  & 3.11  & 3.11  & 2.38  & 4.34  & 2.73  & 6.37  & 15.22 \\
    Glass. & 27.34 & \textbf{26.64} & 32.62 & 29.39 & 29.30 & 29.77 & 29.44 & 30.23 & 31.64 & 29.49 & 51.40 & 52.43 \\
    Haberman & 26.44 & 26.99 & \textbf{24.51} & 25.26 & 29.90 & 26.70 & 25.36 & 25.88 & 26.37 & 26.99 & 25.20 & 26.24 \\
    Heart & 16.81 & 19.04 & \textbf{15.07} & 15.74 & 19.11 & 16.85 & 16.85 & 17.04 & 17.11 & 17.04 & 15.56 & 19.74 \\
    ILPD  & 30.29 & 29.97 & \textbf{28.01} & 28.70 & 31.30 & 30.89 & 30.65 & 29.35 & 28.95 & 29.02 & 44.10 & 49.55 \\
    Image & 4.58  & 3.94  & 6.10  & 4.03  & 3.70  & \textbf{3.63} & \textbf{3.63} & 5.90  & 5.80  & 5.87  & 20.39 & 23.33 \\
    Iris  & \textbf{3.80} & 4.07  & 4.00  & 4.60  & 5.47  & 4.07  & 3.87  & 3.87  & 4.60  & 4.40  & 4.60  & 4.07 \\
    Leaf  & 28.50 & 26.68 & 50.41 & 35.56 & 28.35 & \textbf{25.74} & 26.12 & 33.18 & 42.91 & 32.03 & 30.26 & 26.74 \\
    Pageblock & 3.34  & 3.28  & 3.41  & 3.11  & \textbf{3.07} & 3.15  & 3.20  & 3.39  & 3.32  & 3.62  & 9.32  & 9.24 \\
    Parkinsons & 5.85  & 5.85  & 5.85  & \textbf{5.33} & 5.49  & 5.85  & 5.85  & 5.85  & 5.90  & 12.00 & 29.95 & 40.51 \\
    Seeds & 6.71  & 6.86  & 7.00  & 6.33  & \textbf{6.00} & 6.48  & 6.48  & 6.62  & 6.81  & 7.29  & 9.62  & 7.57 \\
    Sonar & 11.68 & \textbf{10.43} & 14.47 & 13.51 & 14.04 & 12.45 & 12.45 & 12.98 & 17.69 & 15.19 & 31.11 & 17.45 \\
    Spambase & 7.34  & 8.01  & 8.62  & 7.65  & 7.71  & 8.07  & 8.11  & 8.44  & 8.91  & \textbf{6.69} & 18.62 & 22.62 \\
    Spectf & 19.74 & 21.80 & 19.48 & 20.75 & 26.74 & 20.30 & 20.30 & 20.19 & \textbf{19.25} & 20.45 & 32.58 & 38.58 \\
    Vehicle & 23.85 & 27.06 & 28.74 & 28.17 & 29.18 & 24.17 & 24.22 & 25.04 & 27.84 & \textbf{23.71} & 54.55 & 42.88 \\
    Vertebral1 & 16.42 & 17.32 & 16.68 & 16.61 & 17.32 & 15.35 & 15.58 & 15.90 & 16.97 & \textbf{15.06} & 22.32 & 23.13 \\
    Vertebral2 & \textbf{15.00} & 16.55 & 19.35 & 19.35 & 18.65 & 18.48 & 18.19 & 17.13 & 19.97 & 16.55 & 17.65 & 16.90 \\
    WBC   & 2.91  & 3.59  & 3.00  & 3.13  & 3.21  & \textbf{2.52} & \textbf{2.52} & 2.96  & 3.07  & 3.13  & 3.79  & 4.08 \\
    WDBC  & \textbf{2.62} & 3.01  & 3.30  & 2.99  & 3.34  & 2.72  & 2.72  & 3.16  & 3.46  & \textbf{2.62} & 6.66  & 9.81 \\
    Wine  & \textbf{0.89} & 1.29  & 2.48  & 2.20  & 4.16  & 1.97  & 1.97  & 1.57  & 1.19  & 1.86  & 3.31  & 3.65 \\
    WinequalityR & \textbf{34.00} & 34.57 & 40.28 & 34.94 & 34.52 & 36.53 & 36.53 & 38.28 & 39.19 & 37.64 & 44.85 & 37.85 \\
    WinequalityW & \textbf{34.45} & 35.42 & 43.64 & 35.41 & 35.28 & 36.14 & 36.14 & 42.39 & 41.50 & 42.57 & 55.62 & 48.04 \\
    \midrule
    Average AMR & \textbf{15.27} & 15.91 & 17.70 & 16.29 & 17.01 & 15.73 & 15.67 & 16.35 & 17.52 & 16.31 & 24.28 & 24.80 \\
    Average rank & \textbf{3.41 } & 5.70  & 7.37  & 5.54  & 7.26  & 4.91  & 4.52  & 5.76  & 7.74  & 5.72  & 9.81  & 10.26  \\
    win/tie/loss &       & 19/1/7 & 21/1/5 & 19/0/8 & 22/0/5 & 19/1/7 & 18/1/8 & 21/1/5 & 22/0/5 & 19/1/7 & 24/0/3 & 25/0/2 \\
    \toprule
       & \multicolumn{12}{c}{AF1} \\
    \midrule
    Blood & 0.6610 & 0.6483 & 0.6529 & 0.6535 & 0.6127 & 0.6695 & \textbf{0.6756} & 0.6594 & 0.6594 & 0.5947 & 0.5564 & 0.5217 \\
    Bupaliver & 0.6691 & 0.6598 & 0.6282 & 0.6323 & 0.6205 & 0.6291 & 0.6321 & 0.6395 & 0.6277 & \textbf{0.6798} & 0.5543 & 0.5282 \\
    Cardio1 & 0.7436 & 0.7240 & 0.6597 & 0.7210 & 0.7288 & 0.7435 & 0.7394 & 0.7205 & 0.6755 & \textbf{0.7638} & 0.6793 & 0.6589 \\
    Cardio2 & 0.8422 & 0.8355 & 0.8225 & 0.8338 & 0.8327 & 0.8469 & \textbf{0.8478} & 0.8239 & 0.8248 & 0.8328 & 0.7086 & 0.6725 \\
    Climate & 0.6399 & 0.7001 & 0.6458 & 0.6502 & 0.6419 & 0.6350 & 0.6350 & 0.6403 & 0.6262 & \textbf{0.8002} & 0.7665 & 0.6836 \\
    Dermatology & \textbf{0.9772} & 0.9574 & 0.9559 & 0.9562 & 0.9555 & 0.9663 & 0.9663 & 0.9734 & 0.9538 & 0.9702 & 0.9268 & 0.8145 \\
    Glass. & 0.6537 & 0.6672 & 0.5889 & 0.6679 & \textbf{0.6730} & 0.6658 & 0.6656 & 0.6139 & 0.6041 & 0.6204 & 0.4748 & 0.4830 \\
    Haberman & 0.5711 & 0.5520 & 0.5871 & 0.5800 & 0.5559 & 0.6230 & \textbf{0.6478} & 0.5748 & 0.5662 & 0.5009 & 0.5749 & 0.5301 \\
    Heart & 0.8299 & 0.8071 & \textbf{0.8446} & 0.8383 & 0.8060 & 0.8288 & 0.8288 & 0.8269 & 0.8257 & 0.8266 & 0.8419 & 0.7990 \\
    ILPD  & \textbf{0.6471} & 0.6410 & 0.6053 & 0.6053 & 0.6085 & 0.6231 & 0.6230 & 0.6053 & 0.6091 & 0.4311 & 0.5589 & 0.5031 \\
    Image & 0.9540 & 0.9605 & 0.9383 & 0.9594 & 0.9629 & \textbf{0.9636} & \textbf{0.9636} & 0.9403 & 0.9414 & 0.9411 & 0.7781 & 0.7462 \\
    Iris  & \textbf{0.9620} & 0.9593 & 0.9600 & 0.9540 & 0.9454 & 0.9593 & 0.9613 & 0.9614 & 0.9540 & 0.9561 & 0.9540 & 0.9593 \\
    Leaf  & 0.7158 & 0.7323 & 0.4507 & 0.6261 & 0.7124 & \textbf{0.7434} & 0.7381 & 0.6551 & 0.5598 & 0.6652 & 0.6995 & 0.7277 \\
    Pageblock & 0.8289 & 0.8292 & 0.8294 & 0.8406 & \textbf{0.8450} & 0.8412 & 0.8385 & 0.8325 & 0.8364 & 0.7714 & 0.5931 & 0.5927 \\
    Parkinsons & 0.9233 & 0.9233 & 0.9233 & 0.9302 & \textbf{0.9275} & 0.9233 & 0.9233 & 0.9233 & 0.9212 & 0.8147 & 0.6807 & 0.5886 \\
    Seeds & 0.9326 & 0.9306 & 0.9296 & 0.9364 & \textbf{0.9398} & 0.9352 & 0.9352 & 0.9336 & 0.9319 & 0.9269 & 0.9036 & 0.9230 \\
    Sonar & 0.8814 & \textbf{0.8940} & 0.8534 & 0.8631 & 0.8571 & 0.8742 & 0.8742 & 0.8686 & 0.8188 & 0.8467 & 0.6875 & 0.8240 \\
    Spambase & 0.9228 & 0.9160 & 0.9094 & 0.9196 & 0.9190 & 0.9154 & 0.9152 & 0.9104 & 0.9058 & \textbf{0.9295} & 0.8129 & 0.7738 \\
    Spectf & \textbf{0.6453} & 0.6360 & 0.6066 & 0.6019 & 0.6012 & 0.6075 & 0.6075 & 0.6116 & 0.6223 & 0.6002 & 0.6406 & 0.5940 \\
    Vehicle & 0.7570 & 0.7251 & 0.7077 & 0.7120 & 0.7028 & 0.7582 & 0.7578 & 0.7447 & 0.7086 & \textbf{0.7621} & 0.4263 & 0.5562 \\
    Vertebral1 & 0.8069 & 0.8016 & 0.8102 & 0.8107 & 0.8065 & 0.8222 & 0.8170 & 0.8165 & 0.8061 & \textbf{0.8256} & 0.7653 & 0.7638 \\
    Vertebral2 & \textbf{0.8082} & 0.7902 & 0.7723 & 0.7718 & 0.7788 & 0.7736 & 0.7786 & 0.7904 & 0.7628 & 0.7931 & 0.7813 & 0.7970 \\
    WBC   & 0.9681 & 0.9604 & 0.9670 & 0.9655 & 0.9647 & \textbf{0.9726} & \textbf{0.9726} & 0.9676 & 0.9661 & 0.9658 & 0.9589 & 0.9551 \\
    WDBC  & 0.9719 & 0.9675 & 0.9642 & 0.9677 & 0.9639 & 0.9707 & 0.9707 & 0.9657 & 0.9624 & \textbf{0.9720} & 0.9284 & 0.8881 \\
    Wine  & \textbf{0.9912} & 0.9882 & 0.9740 & 0.9783 & 0.9633 & 0.9821 & 0.9821 & 0.9886 & 0.9897 & 0.9837 & 0.9714 & 0.9730 \\
    WinequalityR & \textbf{0.3767} & 0.3627 & 0.2959 & 0.3601 & 0.3655 & 0.3620 & 0.3620 & 0.2948 & 0.3176 & 0.2916 & 0.3309 & 0.3014 \\
    WinequalityW & \textbf{0.4089} & 0.4077 & 0.2985 & 0.3954 & 0.4021 & 0.3961 & 0.3961 & 0.3000 & 0.3263 & 0.2641 & 0.2670 & 0.2726 \\
    \midrule
    Average AF1 & \textbf{0.7811} & 0.7769 & 0.7475 & 0.7678 & 0.7664 & 0.7789 & 0.7798 & 0.7623 & 0.7520 & 0.7530 & 0.6971 & 0.6826 \\
    Average rank & \textbf{3.54 } & 5.35  & 7.80  & 5.81  & 6.67  & 4.17  & 4.20  & 6.11  & 8.17  & 6.59  & 9.30  & 10.30  \\
    win/tie/loss &       & 20/1/6 & 21/1/5 & 18/0/9 & 21/0/6 & 15/1/11 & 15/1/11 & 21/1/5 & 26/0/1 & 20/0/7 & 24/0/3 & 25/0/2 \\
    \bottomrule

    \end{tabular}%
    \flushleft
    1. The kNN-based classifiers are on the left and the non-kNN classifiers are on the right. \\
    2. GME-NBC and KDE-NBC are the naive Bayesian classifiers that estimates the probability density with GME and KDE respectively. \\
  \label{tab:results}%
\end{table*}%

We have employed a Friedman test \cite{hollander1999nonparametric,Dem2006Statistical} for multiple comparison among these classifiers. The computed Friedman statistics for AMR and AF1 are respectively 97.07 and 100.03, both higher than the critical value $\chi^2_{0.05,11} = 19.68$ at the 0.05 significance, which indicates significant difference among these 12 classifiers in terms of both AMR and AF1.

We further use the post-hoc Bonferroni-Dunn test \cite{Dem2006Statistical} to reveal the differences among the classifiers. Fig. \ref{multitest} shows the results of the Bonferroni-Dunn test that the other classifiers are compared to LD-kNN(GME). Fig. \ref{multitest_AMR} indicates that, in terms of AMR, LD-kNN(GME) performs significantly ($p<0.05$) better than KDE-NBC, GME-NBC, SVM-kNN(Poly), V-kNN and DW2-kNN and the data is not sufficient to indicate the differences between LD-kNN(GME) and the other 6 classifiers. Also, Fig. \ref{multitest_AF1} indicates that, in terms of AF1, LD-kNN(GME) can significantly ($p<0.05$) outperform KDE-NBC, GME-NBC, SVM-kNN(Poly), V-kNN, DW2-kNN, SVM and SVM-RBF and the differences between LD-kNN(GME) and the other 4 classifiers can not be detected from the data.

\begin{figure}[!t]
\centering
\subfloat[AMR]{\includegraphics[width=\columnwidth]{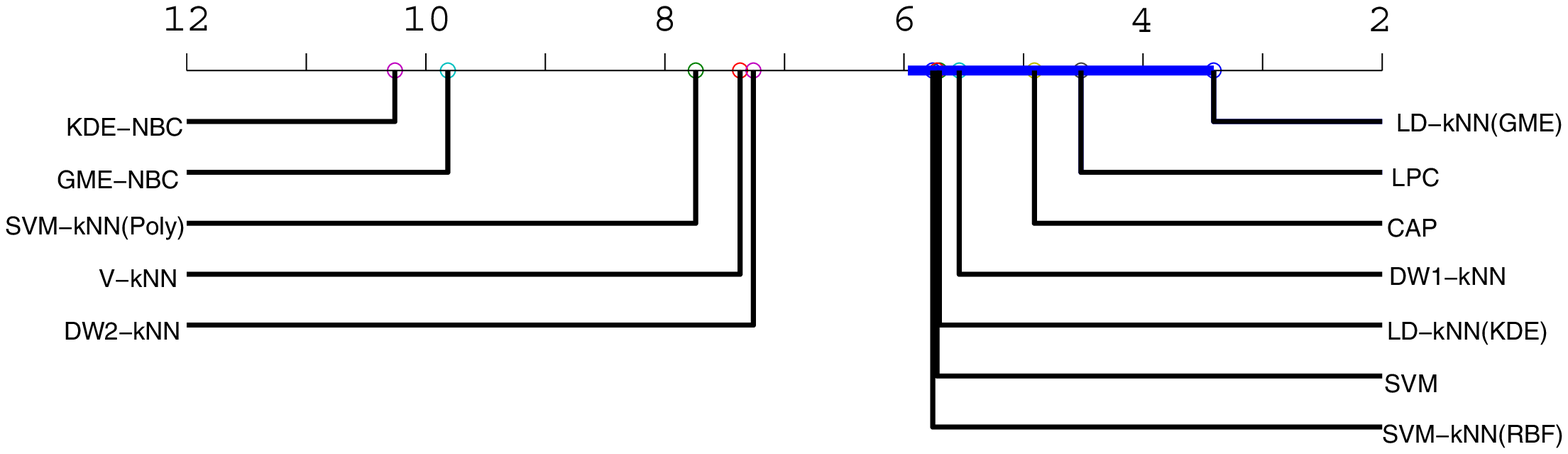}
\label{multitest_AMR}}

\subfloat[AF1]{\includegraphics[width=\columnwidth]{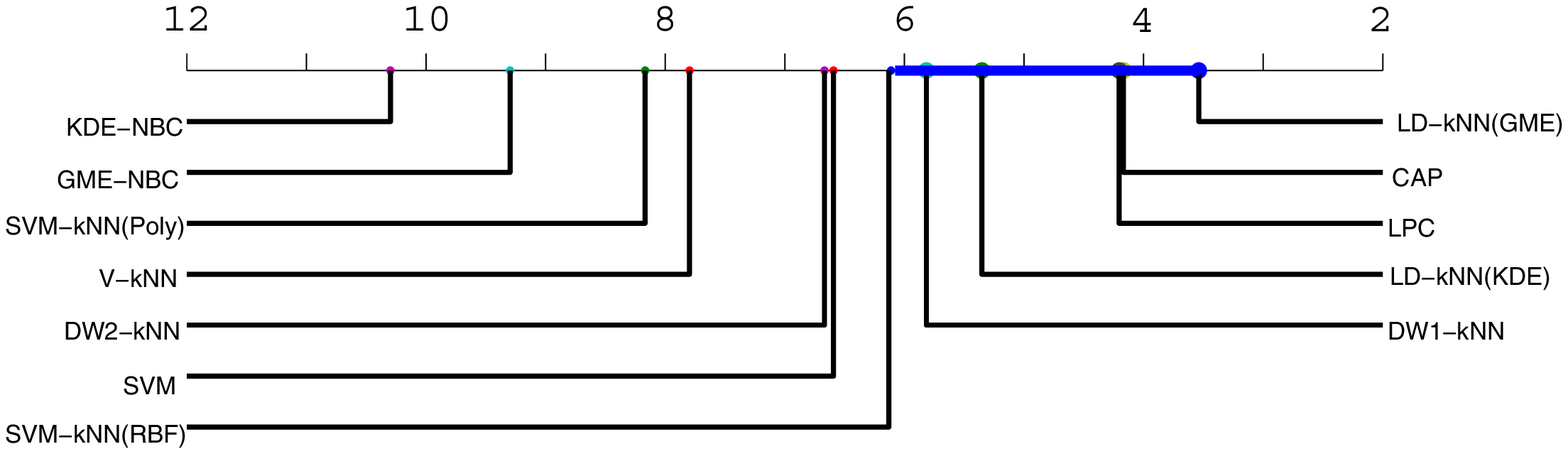}
\label{multitest_AF1}}

\caption{Comparisons of LD-kNN(GME) against the other classifiers with the Bonferroni-Dunn test in terms of (a) AMR and (b) AF1. The LD-kNN(GME) can significantly ($p<0.05$) outperform the classifiers with ranks outside the marked interval.}
\label{multitest}
\end{figure}

Additionally, to evaluate how well a particular method performs on average among all the problems taken into consideration, it is the issue of robustness to classification problems. Following the method designed by Friedman \cite{friedman1994flexible,li2008nearest}, we quantify the robustness of a classifier \textit{m} by the ratio $r_m$ of its error rate $e_m$ to the smallest error rate over all the methods being compared in a particular application, i.e., $r_m=e_m/{\min_{1 \leq {k} \leq 12}{e_k}}$. A greater value for this ratio indicates an inferior performance for the corresponding method for that application among the comparative methods. Thus, the distribution of $r_m$ for each method over all the datasets provides information concerning its robustness to classification problems. We illustrate the distribution of $r_m$ for each method over the 27 datasets by box plots in Fig. \ref{robustness} where it is clear that the spread of $r_m$ for LD-kNN(GME) is narrow and close to point 1.0; this demonstrates that the LD-kNN(GME) method performs robustly to these classification problems. From Fig. \ref{robustness} we can see that, as state-of-the-art classifiers CAP and LPC are also as robust to these classification problems as LD-kNN(GME). LD-kNN(KDE) is less robust than LD-kNN(GME), CAP and LPC, as it is at the medium level in terms of the robustness. NBC performs least well in practical problems; this is due to the fact that the class conditional independence assumption is usually violated in practical problems.

\begin{figure}[!t]
\centering
\includegraphics[width=\columnwidth]{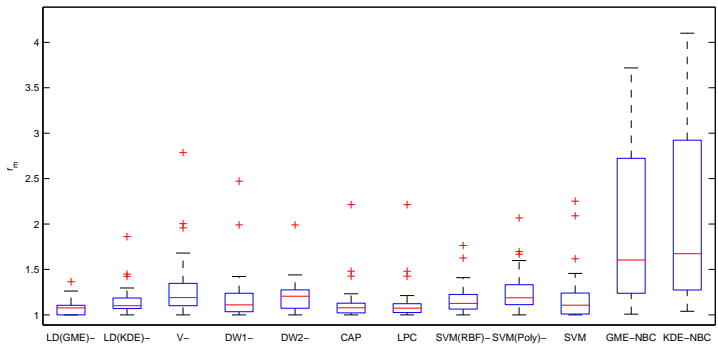}
\caption{The $r_m$ distributions of different classifiers.}
\label{robustness}
\end{figure}

In considering the reported results, it can be concluded that LD-kNN(GME) performs better than most of other classifiers in terms of effectiveness (described by AMR or AF1) and robustness. In considering the kNN-based classifiers, the DW-kNN improves the traditional V-kNN by weighting and achieves a better performance over these datasets. For the LD-kNN(KDE), by estimating the local distribution based on distances, it is essentially a type of DW-kNN method and we have observed similar performances for the LD-kNN(KDE), DW1-kNN and DW2-kNN in terms of robustness as shown in Fig. \ref{robustness}. The CAP and the LPC improve the kNN method by using local center and can be viewed as simplified versions of LD-kNN(GME); and they can achieve a comparable effectiveness and robustness with LD-kNN(GME). In Combining SVM and kNN methods, the SVM-kNN achieves the dimensional scalability; however, it fails to perform as effectively and robustly as the LD-kNN(GME). The LD-kNN(GME) is a more comprehensive method, as it considers the kNNs integrally by local distribution and generates the classification decision rule by maximizing posterior probability; thus it is reasonable to conclude that among the kNN-based classifiers the LD-kNN(GME) should have an upper level performance.

\section{Discussion}\label{sec:discussion}
As a kNN-based classifier, LD-kNN can inherit the advantages of kNN method and moreover, due to the classification
method being based on maximum posterior probability,
the LD-kNN classifier can achieve the Bayes error rate in
theory. However, in practice the effectiveness of LD-kNN
usually depends on the accuracy of LPD estimation. The
LPD is usually estimated through a local model assumption.
In this article, we have introduced two approaches GME
and KDE for LDP estimation from the samples in the
neighborhood. With the intuition that the distribution in a
local region should be simpler than in the whole sample
space, a moderately complex model should fit the local
distribution better. Thus, due to the complexity of KDE and
the simplicity of uniform assumption (as is the case in the
V-kNN), a local model with medium complexity such as LD-kNN(GME) should be more reasonable, and moreover, it has
demonstrated its effectiveness, efficiency and robustness in
the experimental results compared with LD-kNN(KDE) and
kNN.

In fact, the selection of local distribution model is quite
dependent on the size of neighborhood. If the neighborhood
is small, a complex local model created from a small number
of samples may not so stable; a local uniform model
should be more adapted. On the other hand, samples in a
large neighborhood are not always distributed uniformly; a
complex local model is favored. For a sparse or complexly
distributed dataset, a large neighborhood size and a relatively
complex local model should be selected for the LPD
estimation. From the results of Experiment I in Fig. \ref{fig_kpc2} and \ref{fig_kpc_real},
we can see that in a larger neighborhood LD-kNN(GME) is
usually superior to V-kNN, while in a small neighborhood
the superiority is not so significantly. However, the exact
relationship between neighborhood size and the local distribution
model is not available, how to select a neighborhood
size and the corresponding local distribution model remains
a part of our future research.

The LD-kNN is essentially a Bayesian classification
method as it is predicated on the Bayes theorem and predicts
the class membership probabilities. LD-kNN estimates
the posterior probability through the local distribution described
by local probability density around the query sample
rather than through the global distribution as NBC. The
NBC, as a Bayesian classier can also achieve the minimum
error rate in theory. However, the global distribution model
is usually too complex to make a model assumption; thus
the class conditional independence assumption should be
made to reduce the complexity of the global distribution.
Thus, along with the fact that the class conditional independence
assumption usually does not hold in practical
problems, it is not surprising that LD-kNN performs much
more effectively than NBC in Experiment IV.

\section{Conclusion}\label{sec:conclusion}
We have introduced the concept of local distribution and
have considered kNN classification methods based on local
distribution. Subsequently, we have presented our novel
LD-kNN technique developed to enable pattern classification;
the LD-kNN method extends the traditional kNN
method in that it employs local distribution. The LD-kNN
method essentially considers the k nearest neighbors of the
query sample as several integral sets by the class labels
and then organizes the local distribution information in
the area represented by these integral sets to achieve the
posterior probability for each class; then the query sample is
classified to the class with the greatest posterior probability.
This approach provides a simple mechanism for estimating
the probability of the query sample attached to each class
and has been shown to present several advantages. Through
tuning the neighborhood size and the local distribution
model, it can be applied to various datasets and achieve
good performances.

In our experiments we have investigated the properties
of the proposed LD-kNN methods. The results of Experiment
I indicates that LD-kNN(GME) usually favors a
reasonably larger neighborhood size and can achieve an
effective and relatively stable performance in most cases.
Experiment II shows that LD-kNN(GME) can be applied
to a high-dimensional dataset as the NBC and SVM methods;
Experiment III indicates the efficiency of LD-kNN; its
time consumption would not increase too much with the
expansion of the dataset or the neighborhood. Experiment
IV demonstrates the effectiveness and robustness of LD-kNN(GME) to real classification problems. The experimental
results demonstrate the advantages of LD-kNN and
show its potential superiority in a variety of research area.

\section*{Acknowledgment}
This work was supported by the National Basic Research Program of China (2014CB744600), the National Natural Science Foundation of China (61402211, 61063028 and 61210010) and the International Cooperation Project of Ministry of Science and Technology (2013DFA11140).


\bibliographystyle{IEEEtran}
\bibliography{LD}

\end{document}